\documentclass[11pt]{article}
\usepackage[preprint]{acl} 

\usepackage{fvextra}
\usepackage{latexsym}
\usepackage[T1]{fontenc}
\usepackage[utf8]{inputenc}
\usepackage{microtype}
\usepackage{inconsolata}
\usepackage{amsmath,amssymb,amsthm,mathtools}
\usepackage{bm}
\usepackage{booktabs}
\usepackage{multirow}
\usepackage{array}
\usepackage{graphicx}
\usepackage{xcolor}
\usepackage{caption}
\usepackage{times}
\usepackage{tikz}
\usetikzlibrary{arrows.meta,positioning,calc,fit,backgrounds,shapes.geometric}
\usepackage{url}
\DeclareMathOperator{\diag}{diag}

\DeclareMathOperator{\softplus}{softplus}

\DeclareMathOperator{\ECE}{ECE}

\newcommand{\R}{\mathbb{R}}
\newcommand{\E}{\mathbb{E}}
\newcommand{\KL}{\mathrm{KL}}
\newcommand{\cE}{\mathcal{E}}
\newcommand{\cR}{\mathcal{R}}
\newcommand{\cT}{\mathcal{T}}

\newcommand{\cD}{\mathcal{D}}
\newcommand{\cS}{\mathcal{S}}
\newcommand{\cN}{\mathcal{N}}
\newcommand{\cL}{\mathcal{L}}
\newcommand{\cB}{\mathcal{B}}
\newcommand{\cM}{\mathcal{M}}

\newcommand{\eps}{\varepsilon}
\newcommand{\method}{BLADE}

\newtheorem{proposition}{Proposition}
\newtheorem{theorem}{Theorem}

\title{BLADE: Distilled LLM Regularization for Calibrated Knowledge Graph Completion}

\author{
\textbf{Ibne Farabi Shihab}\thanks{Corresponding author: \texttt{ishihab@iastate.edu}.}\textsuperscript{1}
\quad
\textbf{Rabeya Bosri Tamanna}\textsuperscript{3}
\quad
\textbf{Abdo El Karaky}
\\
\textbf{Sanjeda Akter}\textsuperscript{1}
\quad
\textbf{Anuj Sharma}\textsuperscript{2}
\\[4pt]
{\small \textsuperscript{1}Department of Computer Science, Iowa State University}\\
{\small \textsuperscript{2}Department of Civil, Construction \& Environmental Engineering, Iowa State University}\\
{\small \textsuperscript{3}Brac University}
}

\begin{document}
\maketitle
\begin{abstract}
Knowledge graph completion models optimize ranking, although many downstream uses require calibrated probabilities. We present \method, a variational model that separates latent truth from graph recording and distills offline language-model judgments into a frozen teacher regularizer. The LLM is absent at inference; posterior samples provide predictive probabilities and epistemic uncertainty, while the compact teacher remains available only as an optional triage factor. Across five benchmarks, \method{} remains competitive under a common ranking protocol and reduces adaptive ECE by a macro-average of 60.1\% relative to deep ensembles and 78.1\% relative to temperature-scaled RotatE. On identical FB15k-237 candidates, it also improves ECE, Brier score, and NLL over validation-selected histogram binning and a matched generative ComplEx$^2$ model, with the advantage persisting on a prespecified near-miss pool. Under controlled injected missingness, the full triage score reaches 0.863 mean AUC-PR versus 0.805 for its strongest non-teacher variant. Leakage stress tests show that aligned semantics matter but cannot exclude knowledge acquired during LLM pretraining. We therefore claim calibration only for the declared candidate distributions, not for all unobserved triples.
\end{abstract}

\section{Introduction}

Knowledge graph completion (KGC) infers missing facts from an incomplete graph. Embedding models such as TransE \citep{bordes2013translating}, ComplEx \citep{trouillon2016complex}, ConvE \citep{dettmers2018conve}, and RotatE \citep{sun2019rotate} make this problem tractable by learning scores over triples. These scores are effective for ordering candidate entities, yet their numerical values need not be probabilities. A curator deciding which predicted facts to verify, or a reasoning system deciding whether to consume them, needs to know whether a confidence of 0.8 is correct about 80\% of the time on the candidates it will actually encounter.

The distinction between rank and probability becomes sharper under the open-world assumption. A triple absent from a graph can be false, structurally ambiguous, or true but unrecorded. Standard negative sampling provides a useful ranking surrogate, but it does not by itself distinguish latent truth from graph inclusion. Moreover, calibration is inseparable from the candidate distribution: a model calibrated on random corruptions may fail on plausible near misses. We therefore evaluate calibration on fully specified binary pools and treat transfer to a second, model-mined pool as a particular distribution shift rather than evidence of distribution-free calibration.

Language models offer complementary semantic evidence when graph structure is sparse \citep{yao2019kgbert,wang2021kepler}. Directly querying them at inference, however, introduces latency, nondeterminism, and the possibility of hallucinated or temporally mismatched facts \citep{wei2023kicgpt,zhang2024kopa}. It also obscures whether improvements come from graph learning or memorized benchmark semantics. A more controlled design is to query a language model once, distill its scores into a compact teacher, freeze that teacher, and measure what information survives systematic leakage stress tests.

\method{} follows this design in three stages. Offline language-model labels train a frozen plausibility teacher. During KGC training, a relational encoder parameterizes factorized Gaussian embeddings, a probabilistic decoder maps samples to latent truth probabilities, an observation factor represents graph recording, and the teacher supplies a bounded energy regularizer. At inference, the language model and teacher-labeling pipeline disappear; posterior samples yield ranking scores and calibrated probabilities, while the compact frozen teacher can optionally modify hidden-positive triage. Because the experimental teacher candidates are constructed from the benchmark training graph, the LLM term is data dependent and is consistently called a regularizer. A prior interpretation is valid only in the separate regime where every prior-defining quantity is fixed independently of the sample being analyzed.

The contribution lies in this separation and joint evaluation of roles, not in claiming the first probabilistic KGE, variational embedding, knowledge-distillation method, or calibrated link predictor. Empirically, we compare against post-hoc calibration, uncertainty wrappers, a matched generative ComplEx$^2$ construction, and both RotatE- and ComplEx-backed versions of \method{}. The resulting evidence supports three restrained conclusions: ranking quality is largely preserved; calibration improves on the specified exact and near-miss candidate pools; and the recording-propensity factor and aligned teacher signal improve recovery of deliberately hidden positives. A conditional local-precision result explains what a convex teacher energy can change, while the limitations make explicit the mean-field posterior, non-identifiability of recording propensity, candidate-shift dependence, and possible pretrained semantic leakage.
\section{Related work}

Classical KGE models learn unnormalized compatibility scores for transductive link prediction \citep{bordes2013translating,trouillon2016complex,dettmers2018conve,sun2019rotate}. Mapping these scores to probabilities requires both a declared negative distribution and a fitted calibration protocol \citep{guo2017calibration}. This requirement motivates our matched Platt, isotonic, beta, and histogram-binning comparisons. A normalized generative ComplEx$^2$ construction provides a different probabilistic object, so we compare ordinary ComplEx, calibrated ComplEx, and a matched ComplEx$^2$ model on identical candidate identifiers rather than mixing ECE values from incompatible pools.

Distributional models such as KG2E \citep{he2015kg2e} and TransG \citep{xiao2015transg}, together with Bayesian graph uncertainty work \citep{hasanzadeh2020bayesian}, motivate uncertainty over entity and relation representations. MC dropout \citep{gal2016dropout} and deep ensembles \citep{lakshminarayanan2017simple} provide model-agnostic predictive uncertainty, whereas conformal KGE methods target set coverage rather than scalar probability calibration. \method{} uses a factorized variational density whose embeddings share information through a relational encoder. This coupling does not create an unrestricted joint posterior, a distinction we retain throughout the paper.

Text-enhanced systems including KG-BERT \citep{yao2019kgbert}, KEPLER \citep{wang2021kepler}, and SimKGC \citep{wang2022simkgc}, as well as LLM-assisted systems such as KICGPT \citep{wei2023kicgpt}, KoPA \citep{zhang2024kopa}, and CP-KGC \citep{yang2023cpkgc}, demonstrate the value of semantic side information. Our teacher follows the classical distillation principle of transferring a larger model's behavior to a smaller frozen model, but its purpose is regularization during probabilistic KGC training rather than an inference-time text scorer. Appendix~\ref{app:extended-related-work} places these methods, neural rankers, generative KGE, and foundation-style KG reasoners within their distinct evaluation protocols.

\section{Problem setup}

Let \(\cE\) and \(\cR\) denote the entity and relation sets, and let \(\cT_{\mathrm{obs}}\subset\cE\times\cR\times\cE\) be the recorded training graph. For each candidate \(i=(h,r,t)\), the latent variable \(y_i\in\{0,1\}\) indicates truth and \(o_i\in\{0,1\}\) indicates whether the graph records the triple. The open-world assumption is precisely that \(o_i=0\) does not imply \(y_i=0\).

The model is evaluated through three outputs that should not be conflated. Filtered ranking orders all admissible completions of a query and is measured by MRR and Hits@\(k\). Calibration concerns a binary candidate distribution \(Q\) over labeled triples and asks whether \(\Pr_Q(y=1\mid \hat p\approx c)\approx c\); we report adaptive ECE, Brier score, and NLL. Hidden-positive triage removes a relation-stratified set of true training triples before fitting, then ranks those unobserved positives among relation-matched negatives using AUC-PR and precision at a fixed review budget. The exact construction of each candidate pool is part of the estimand, not an incidental implementation choice. Appendix~\ref{app:calibration} defines the standard pool and the harder prespecified near-miss shift, while Appendix~\ref{app:fn-protocol} defines injected missingness.

\section{Model}

Figure~\ref{fig:architecture} separates the pipeline into offline distillation, teacher-regularized KGC training, and LLM-free inference. This separation is operational as well as conceptual: the offline stage creates fixed scores, the training stage assigns a distinct role to each probabilistic component, and the final stage consumes only frozen learned models rather than issuing LLM queries.

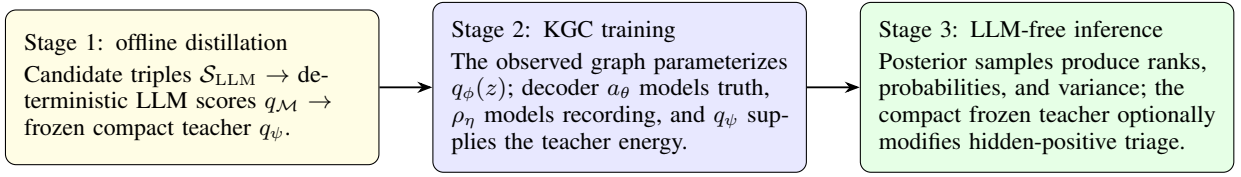
\begin{figure*}[t]
\centering
\begin{tikzpicture}[
    node distance=0.7cm,
    stage/.style={rectangle, draw, rounded corners, minimum height=2.05cm,
        text width=4.45cm, align=left, font=\small, inner sep=7pt},
    arrow/.style={-{Stealth[length=2mm]}, line width=0.7pt}
]
\node[stage, fill=yellow!12] (offline) {Stage 1: offline distillation\\[2pt]
Candidate triples $\cS_{\mathrm{LLM}}$ $\rightarrow$ deterministic LLM scores $q_{\cM}$ $\rightarrow$ frozen compact teacher $q_\psi$.};
\node[stage, fill=blue!9, right=of offline] (training) {Stage 2: KGC training\\[2pt]
The observed graph parameterizes $q_\phi(z)$; decoder $a_\theta$ models truth, $\rho_\eta$ models recording, and $q_\psi$ supplies the teacher energy.};
\node[stage, fill=green!10, right=of training] (inference) {Stage 3: LLM-free inference\\[2pt]
Posterior samples produce ranks, probabilities, and variance; the compact frozen teacher optionally modifies hidden-positive triage.};
\draw[arrow] (offline) -- (training);
\draw[arrow] (training) -- (inference);
\end{tikzpicture}
\caption{The three stages of \method. Only Stage 1 queries the language model. The benchmark experiments use the distilled scores as a data-dependent training regularizer; inference uses neither LLM calls nor teacher relabeling.}
\label{fig:architecture}
\end{figure*}

\subsection{Latent truth and graph recording}

Each entity \(e\in\cE\) and relation \(r\in\cR\) has a latent embedding \(z_e,z_r\in\R^d\). In the RotatE-backed model, the two halves of each embedding form real and imaginary components, and a triple receives the logit
\begin{equation}
    a_\theta(h,r,t;z)
    = \alpha_r\left(\gamma_r - \lVert z_h \circ z_r - z_t\rVert_2\right),
    \label{eq:decoder-logit}
\end{equation}
where \(\gamma_r\in\R\) is a relation-specific margin and \(\alpha_r=\softplus(\tilde\alpha_r)>0\) is an inverse temperature. Unlike directly applying a sigmoid to a nonpositive distance, this parameterization can express probabilities above one half, addressing a basic score-range obstacle to calibration \citep{guo2017calibration}. The matched ComplEx control replaces only this logit; the remaining pipeline is unchanged.

The latent truth probability is then
\begin{equation}
    p_\theta(y_i=1\mid z)=\sigma(a_\theta(i;z)),
    \label{eq:truth-prob}
\end{equation}
The sigmoid of this logit models latent truth. Recording is modeled separately:
\begin{equation}
\begin{aligned}
    p_\eta(o_i=1\mid y_i=1)&=\rho_\eta(i),\\
    p(o_i=1\mid y_i=0)&=0.
\end{aligned}
    \label{eq:observation-model}
\end{equation}
Marginalizing \(y_i\) yields
\begin{equation}
\begin{aligned}
    p_{\theta,\eta}(o_i=1\mid z)&=\rho_\eta(i)\sigma(a_\theta(i;z)),\\
    p_{\theta,\eta}(o_i=0\mid z)&=1-\rho_\eta(i)\sigma(a_\theta(i;z)).
\end{aligned}
\label{eq:observed-likelihood}
\end{equation}
We use a deliberately low-capacity recording model based on relation and degree features,
\begin{equation}
\begin{aligned}
    \rho_\eta(h,r,t)
    =\sigma\!\big(&\beta_0+\beta_r+\beta_h\log(1+\deg(h))\\
    &+\beta_t\log(1+\deg(t))\big).
\end{aligned}
\label{eq:rho-param}
\end{equation}
The parameters are regularized and selected on injected-missing validation splits. Setting \(\rho\equiv1\) recovers the closed-world ablation. Because only the product \(\rho_i p_i\) is observed, truth and recording propensity are not nonparametrically identifiable from one incomplete graph. Accordingly, \(\rho\) is a regularized proxy used for controlled triage, not an estimated causal missingness mechanism.
\subsection{Variational posterior}

The approximate posterior is a factorized Gaussian over entity and relation embeddings,
\begin{equation}
\resizebox{\linewidth}{!}{$
    q_\phi(z\mid\cT_{\mathrm{obs}})
    =\prod_{e\in\cE}\cN\!\left(z_e;\mu_e,\diag(\sigma_e^2)\right)
     \prod_{r\in\cR}\cN\!\left(z_r;\mu_r,\diag(\sigma_r^2)\right),
    \label{eq:posterior}
$}
\end{equation}
The parameters \((\mu,\sigma)\) are produced by a relational encoder that aggregates typed incoming and outgoing messages. Thus different predictions share graph-conditioned representations, even though the variational density itself cannot express arbitrary covariance among embeddings. The exact message equations, neighbor-sampling modes, and cache behavior are given in Appendix~\ref{app:encoder} so that the main exposition can focus on the probabilistic roles of the components.

\subsection{Distilled LLM plausibility}

The offline stage builds \(\cS_{\mathrm{LLM}}\) from observed training positives, type-compatible corruptions, and random corruptions, while excluding benchmark validation and test triples. A deterministic prompt asks a local LLM for one plausibility value in \([0,1]\). Model identifiers, decoding settings, prompts, parser outcomes, retry counts, and rejected responses are retained, making the labeling stage separable from KGC training. Appendix~\ref{app:teacher-config} gives the full domain, range, and sampling construction.

Let \(q_{\cM}(i)\) be the parsed LLM score. A compact teacher \(q_\psi(i)\) predicts this value from entity text, relation text, and type features. It is fitted once per dataset, frozen across all KGC seeds, and checked on a disjoint teacher split using MSE, AUC, ECE, and parser rejection rate; the complete diagnostic set appears in Appendix~\ref{app:teacher-diagnostics}. These diagnostics establish that distillation is neither random nor perfect; they do not establish that the LLM's factual knowledge is free of benchmark leakage.

The teacher then contributes the bounded energy
\begin{equation}
\begin{aligned}
    R_{\psi,\theta}(z)
    &=
    \frac{1}{|\cS_{\mathrm{LLM}}|}
    \sum_{i\in\cS_{\mathrm{LLM}}}
    \Big[
    q_\psi(i)\log \sigma(a_\theta(i;z))\\
    &\quad+
    (1-q_\psi(i))\log(1-\sigma(a_\theta(i;z)))
    \Big].
\end{aligned}
\label{eq:teacher-energy}
\end{equation}
This cross-entropy is bounded above by zero and is evaluated under posterior samples. It aligns the structural predictor with the frozen teacher during training without requiring either the LLM or new teacher labels at KGC inference.

\subsection{Teacher regularization and the prior regime}

The benchmark experiments use \(R_{\psi,\theta}\) as a data-dependent regularizer because \(\cS_{\mathrm{LLM}}\) is derived partly from the training graph. A separate prior interpretation is available only under stronger independence conditions. Let \(w=(z,\theta)\) denote the complete hypothesis and let \(p_0(w)\) be a proper base prior. If \(\cS_{\mathrm{LLM}}\), \(q_\psi\), and every other quantity defining the energy are fixed independently of the analyzed sample, then
\begin{equation}
    p_\lambda(w)=\frac{1}{Z_\lambda}p_0(w)\exp\{\lambda R_\psi(w)\}.
    \label{eq:llm-prior}
\end{equation}
Because \(R_\psi(w)\leq0\), the exponential factor is at most one for \(\lambda\geq0\), so the proper base prior guarantees a finite normalizer. Appendix~\ref{app:theory} states the corresponding PAC--Bayes result. It does not apply to the benchmark-dependent teacher used in our experiments.

\subsection{Training objective}

For a minibatch \(\cB^+\) of observed triples and sampled unobserved candidates \(\cB^-\), write \(s_i(z)=\sigma(a_\theta(i;z))\). The observed-data reconstruction objective is
\begin{equation}
\begin{aligned}
    \cL_{\mathrm{obs}}(\theta,\eta,\phi)
    &=\E_{q_\phi(z\mid\cT_{\mathrm{obs}})}\Bigg[\\[-2pt]
    &\quad\sum_{i\in\cB^+}\log\!\left(\rho_\eta(i)s_i(z)\right)\\
    &\quad+\sum_{j\in\cB^-}\log\!\left(1-\rho_\eta(j)s_j(z)\right)
    \Bigg].
\end{aligned}
\label{eq:obs-objective}
\end{equation}
The KL term against the Gaussian base prior is available in closed form,
\begin{equation}
\resizebox{\linewidth}{!}{$
    \KL(q_\phi\Vert p_0)
    =\tfrac{1}{2}\sum_v\sum_{m=1}^d
    \left(\mu_{v,m}^2+\sigma_{v,m}^2-1-\log\sigma_{v,m}^2\right),
    \label{eq:gaussian-kl}
$}
\end{equation}
where \(v\) ranges over entities and relations. The practical objective is
\begin{equation}
\resizebox{\linewidth}{!}{$
    \max_{\theta,\eta,\phi}\;
    \cL_{\mathrm{obs}}(\theta,\eta,\phi)
    -\KL(q_\phi\Vert p_0)
    +\lambda\,\E_{q_\phi}[R_{\psi,\theta}(z)].
    \label{eq:practical-objective}
$}
\end{equation}
Equation~\ref{eq:practical-objective} is the regularized objective used in all benchmark experiments; we do not describe it as an ELBO under a data-independent LLM prior. It is optimized with the reparameterization \(z_v=\mu_v+\sigma_v\odot\eps\), where \(\eps\sim\cN(0,I)\). In the separate independent-teacher regime, a distribution over the complete hypothesis \(w\) can instead be compared with Eq.~\ref{eq:llm-prior}, as formalized in Appendix~\ref{app:theory}.

\subsection{Posterior prediction and triage}

Given the trained variational posterior, the predictive truth probability for a candidate triple \(i\) is estimated with \(S\) samples,
\begin{equation}
\resizebox{\linewidth}{!}{$
    \hat p_i=\frac{1}{S}\sum_{s=1}^S \sigma(a_\theta(i;z^{(s)})),\qquad z^{(s)}\sim q_\phi(z\mid\cT_{\mathrm{obs}}),
    \label{eq:posterior-mean}
$}
\end{equation}
and epistemic uncertainty is the variance of the sampled truth probabilities,
\begin{equation}
    u_i=\frac{1}{S-1}\sum_{s=1}^S\left(\sigma(a_\theta(i;z^{(s)}))-\hat p_i\right)^2.
    \label{eq:posterior-var}
\end{equation}
For easier interpretation we also use the normalized confidence factor \(c_i=1-4u_i\), clipped to \([0,1]\), since the variance of a Bernoulli probability lies in \([0,1/4]\); this avoids treating a variance change from 0.00 to 0.10 as a negligible change in confidence. For false-negative triage among unobserved triples, the posterior-weighted triage score is
\begin{equation}
    \widehat{\mathrm{FN}}(i)
    =\frac{\hat p_i(1-\hat\rho_i)}{1-\hat p_i\hat\rho_i+\epsilon}\,c_i\,\tilde q_\psi(i),
    \label{eq:fn-score}
\end{equation}
where \(\tilde q_\psi(i)\) is the probability from the compact frozen teacher and \(\epsilon\) is a numerical stabilizer. Ranking and calibration use \(\hat p_i\) without this teacher factor; only the optional triage score evaluates \(q_\psi\), and it never calls the LLM. The first factor is the exact conditional probability of latent truth given non-observation when \(p_i\) and \(\rho_i\) are known. The confidence and teacher factors are heuristic ranking modifiers, so Eq.~\ref{eq:fn-score} is evaluated as a triage score rather than presented as a generally Bayes-optimal posterior.

\section{Theoretical scope}
\label{sec:theory}

The theory is deliberately conditional. It explains the local effect of the teacher energy and the algebra behind the triage factor, but it does not turn a data-dependent benchmark teacher into a Bayesian prior or identify the true recording process.

\begin{proposition}[Local precision ordering]
At a fixed expansion point \(z^\star\), let \(A\succ0\) be the Hessian of the negative log likelihood plus the Gaussian penalty, and suppose the Hessian \(H_R\) of the negative teacher energy is positive semidefinite. The corresponding Laplace covariance approximations
\begin{equation}
    \Sigma_0=A^{-1},\qquad
    \Sigma_\lambda=(A+\lambda H_R)^{-1}
\end{equation}
satisfy \(\Sigma_\lambda\preceq\Sigma_0\) for every \(\lambda\geq0\).
\end{proposition}

\begin{proof}
Since \(A+\lambda H_R\succeq A\succ0\), inversion reverses the Loewner order. The statement compares local quadratic approximations at the same point; it neither supplies an unsupported directional rate nor guarantees improved calibration for an inaccurate teacher.
\end{proof}

The observation model gives a separate exact identity that motivates the first factor of Eq.~\ref{eq:fn-score}.

\begin{proposition}[Posterior-inspired false-negative probability]
Suppose \(p_i=p(y_i=1\mid\cD)\) and \(\rho_i=p(o_i=1\mid y_i=1)\) are known. For an unobserved triple \(o_i=0\),
\begin{equation}
    p(y_i=1\mid o_i=0,\cD)=\frac{p_i(1-\rho_i)}{1-p_i\rho_i}.
    \label{eq:true-fn-posterior}
\end{equation}
\end{proposition}

\begin{proof}
By Bayes' rule, \(p(y_i=1\mid o_i=0,\cD)=p(o_i=0\mid y_i=1)p(y_i=1\mid\cD)/p(o_i=0\mid\cD)\). The numerator is \((1-\rho_i)p_i\). The denominator is \(1-p(o_i=1\mid\cD)=1-\rho_i p_i\), which yields Eq.~\ref{eq:true-fn-posterior}.
\end{proof}

Appendix~\ref{app:theory} gives the PAC--Bayes statement available when Eq.~\ref{eq:llm-prior} is defined over the complete hypothesis independently of an i.i.d. analyzed sample. That result is not invoked for the transductive benchmark experiments.
\begin{table*}[t]
\centering
\caption{Filtered MRR under the common rerun protocol. All rows use the same splits and inference scope. Full Hits@1/3/10 results from the same prediction files appear in Appendix~\ref{app:hits}.}
\label{tab:link-main}
\small
\begin{tabular}{lccccc}
\toprule
Method & FB15k-237 & WN18RR & CoDEx-M & NELL-995 & Hetionet \\
\midrule
RotatE & .338 & .476 & .310 & .492 & .210 \\
SimKGC & .368 & .524 & .354 & .521 & .230 \\
KG2E & .291 & .410 & .265 & .415 & .180 \\
Deep ensemble & .349 & .488 & .321 & .504 & .219 \\
\midrule
\method{} & .365 $\pm$ .002 & .518 $\pm$ .003 & .348 $\pm$ .002 & .522 $\pm$ .004 & .243 $\pm$ .002 \\
\bottomrule
\end{tabular}
\end{table*}

\begin{table*}[t]
\centering
\caption{Calibration on each graph's declared binary candidate distribution. Cells give adaptive ECE/Brier; lower is better.}
\label{tab:calibration-main}
\small
\begin{tabular}{lccccc}
\toprule
Method & FB15k-237 & WN18RR & CoDEx-M & NELL-995 & Hetionet \\
\midrule
RotatE + temperature & .082 / .142 & .071 / .120 & .089 / .151 & .091 / .155 & .065 / .118 \\
MC-dropout RotatE & .074 / .131 & .065 / .114 & .081 / .142 & .079 / .139 & .058 / .109 \\
Deep ensemble & .045 / .102 & .038 / .091 & .050 / .112 & .052 / .118 & .034 / .082 \\
KG2E & .112 / .185 & .104 / .171 & .121 / .198 & .118 / .192 & .098 / .161 \\
\method{} without teacher & .049 / .110 & .044 / .098 & .055 / .119 & .058 / .122 & .041 / .089 \\
\midrule
\method{} & .018 / .061 & .014 / .048 & .021 / .068 & .025 / .072 & .011 / .039 \\
\bottomrule
\end{tabular}
\end{table*}

\section{Experimental design}

The evaluation separates ranking preservation, candidate-conditional calibration, controlled hidden-positive triage, component effects, and leakage stress tests. Methods share candidate identifiers and metric code whenever their outputs are compared directly; external published numbers are confined to a provenance table in Appendix~\ref{app:provenance}.

\subsection{Datasets}

We use FB15k-237 \citep{toutanova2015observed}, WN18RR \citep{dettmers2018conve}, CoDEx-M \citep{safavi2020codex}, NELL-995 \citep{xiong2017deeppath}, and Hetionet \citep{himmelstein2017hetionet}, spanning general, lexical, web-extracted, and biomedical graphs. Table~\ref{tab:datasets} reports the exact splits.

\subsection{Baselines}

The common ranking reruns are RotatE, KG2E, a five-model RotatE ensemble, and SimKGC. Calibration baselines use the same binary pools and include temperature scaling, MC dropout, deep ensembles, global and relation-wise Platt scaling, beta calibration, isotonic regression, and equal-mass histogram binning. Source-query-clustered validation selects the post-hoc family and its hyperparameters before test evaluation. To compare with normalized generative KGE, we evaluate ordinary ComplEx, relation-wise calibrated ComplEx, and a matched ComplEx$^2$ model trained by pseudo-log-likelihood after distilled initialization. A ComplEx-backed \method{} control changes only the decoder. Appendix~\ref{app:calibrator-selection} gives the complete selection path and Appendix~\ref{app:provenance} separates these reruns from external references such as NBFNet \citep{zhu2021nbfnet}.

\subsection{LLM teacher protocol}

The offline labeler is LLaMA-3-70B-Instruct \citep{dubey2024llama}, run locally through vLLM with temperature 0.0. Query counts range from 20{,}000 to 60{,}000, with 20\% observed positives, 40\% type-compatible negatives, and 40\% random negatives. A strict parser rejects multiple numbers, out-of-range values, and explanations. The resulting teacher is fitted once per dataset and frozen across all five KGC seeds. Appendix~\ref{app:teacher-config} reports query and compute counts, while Appendix~\ref{app:prompt} reproduces the prompt and parser.

\subsection{Leakage and robustness controls}

We call the semantic controls leakage stress tests rather than proofs of leakage absence. The strict condition excludes validation and test triples from the teacher pool. Description-free prompts retain entity and relation names but remove descriptions; anonymized prompts replace entity names with stable tokens; relation-only prompts retain only relation names and coarse type constraints. Finally, within-relation permutation preserves each relation's teacher-score distribution while destroying triple-level alignment. These interventions localize useful semantic inputs, but none can erase facts already encoded during LLM pretraining.

\subsection{Metrics}

Link prediction uses filtered MRR and Hits@1/3/10. Calibration uses adaptive-bin ECE, Brier score, and NLL on the exact binary pool and on a frozen hard near-miss pool. Hidden-positive triage uses AUC-PR and \(P@500\). For the headline ECE reductions, we first compute the relative reduction within each dataset and then macro-average the five percentages; averaging ECE values before taking the ratio answers a different question.

All performance experiments use five seeds. Each of 10{,}000 paired bootstrap replicates first resamples matched seed indices and then resamples task-specific source clusters: complete filtered rankings for link prediction, a source query with its associated candidates for calibration, and a hidden positive with its matched negatives for triage. We report percentile intervals and apply Holm correction within predeclared comparison families. This two-stage design preserves candidate dependence within a query; Appendix~\ref{app:stats} gives the complete family definitions.

\section{Results}

All point estimates are means over five seeds. Standard deviations are shown where supplied by the matched reruns, and inferential claims follow the paired two-stage bootstrap described above.

\subsection{Link prediction}

Table~\ref{tab:link-main} contains only common-protocol reruns. \method{} trails SimKGC by 0.003--0.006 MRR on FB15k-237, WN18RR, and CoDEx-M, is effectively tied on NELL-995, and leads by 0.013 on Hetionet. The full Hits@\(k\) table in Appendix~\ref{app:hits} has the same pattern. Paired tests support an improvement over RotatE on all five datasets and over SimKGC only on Hetionet; no other comparison is described as significant. Ranking is therefore a preservation criterion, not a state-of-the-art claim.

\subsection{Calibration}
\begin{table*}[h]
\centering
\caption{Matched FB15k-237 calibration. Every method is frozen before evaluation on the exact test pool and the model-mined hard pool. Values are means over five seeds; lower is better. Full standard deviations and all calibrators appear in Appendix~\ref{app:matched-calibration}.}
\label{tab:matched-calibration}
\small
\setlength{\tabcolsep}{4.2pt}
\begin{tabular}{lcccccc}
\toprule
& \multicolumn{3}{c}{Exact candidate pool} & \multicolumn{3}{c}{Hard near-miss pool} \\
\cmidrule(lr){2-4}\cmidrule(lr){5-7}
Method & ECE & Brier & NLL & ECE & Brier & NLL \\
\midrule
RotatE + histogram binning & .023 & .071 & .244 & .052 & .105 & .349 \\
ComplEx$^2$ + relation-wise Platt & .022 & .068 & .232 & .047 & .099 & .331 \\
\method{}--ComplEx & .021 & .065 & .226 & .041 & .091 & .309 \\
\method{}--RotatE & .018 & .061 & .214 & .036 & .085 & .289 \\
\bottomrule
\end{tabular}
\end{table*}
Table~\ref{tab:calibration-main} first compares uncertainty baselines across all five declared binary pools. \method{} reduces adaptive ECE by a macro-average of 60.1\% relative to deep ensembles and 78.1\% relative to temperature-scaled RotatE. Its no-teacher variant remains better than temperature scaling but worse than the ensemble on four of five graphs, indicating that the observation model and posterior help while the semantic regularizer supplies the largest increment.

The closer prior-work comparison in Table~\ref{tab:matched-calibration} uses identical FB15k-237 candidates. Histogram binning is selected by source-query-clustered out-of-fold validation NLL, then frozen. ComplEx$^2$ is matched in embedding dimension and candidate pool. On the exact pool, \method{}--RotatE improves over selected histogram binning by 0.005 ECE (95\% CI [0.002, 0.008], Holm-adjusted \(p=.003\)), 0.010 Brier ([0.006, 0.014], \(p=.002\)), and 0.030 NLL ([0.017, 0.043], \(p=.002\)). All methods degrade on the hard pool without refitting, but the corresponding advantages grow to 0.016 ECE, 0.020 Brier, and 0.060 NLL, with adjusted \(p=.002\) for each. This supports robustness to the specified near-miss shift, not arbitrary deployment shift.

\subsection{False-negative detection}

Table~\ref{tab:fn-main} evaluates controlled injected missingness. The posterior mean alone obtains 0.722 macro AUC-PR. Adding the recording-propensity factor raises it to 0.805, and the full score reaches 0.863, a further gain of 5.8 points. The full model improves on the mean-plus-observation variant on every graph under the paired bootstrap. These results concern recovery of deliberately hidden positives under the declared matching protocol; they do not identify real-world KG missingness.

\subsection{Component roles and leakage stress tests}

Table~\ref{tab:ablation} isolates the training components on FB15k-237 and Hetionet. Removing the observation model has only a small, non-significant MRR effect but clearly worsens ECE and hidden-positive AUC-PR. Replacing distributions by point embeddings degrades all three metrics, while removing the teacher leaves a functional structural model but forfeits much of the calibration and triage gain. These outcomes match the intended division of labor: the decoder supplies compatibility, the variational density supplies epistemic dispersion, the recording proxy reshapes unobserved candidates, and the teacher contributes aligned semantic evidence.

The decoder control reaches .359 MRR, .021 ECE, and .844 false-negative AUC-PR with ComplEx, compared with .365, .018, and .852 with RotatE. Most of the gain therefore transfers to one additional decoder, although we do not claim decoder independence. The semantic controls in Appendix~\ref{app:leakage-results} provide a complementary result. Removing descriptions causes a modest decline, anonymizing entity strings produces a larger drop, and retaining only relation semantics leaves a weaker but nonzero signal. Permuting scores within relations collapses ECE and triage toward chance despite preserving their marginal distributions. This pattern shows that aligned semantics matter; it does not reveal whether those semantics were learned from benchmark-adjacent pretraining data.

\section{Conclusion}

\method{} organizes LLM-assisted KGC around a simple separation: offline semantics regularize training, graph structure determines the variational predictor, recording propensity is modeled explicitly but cautiously, and inference is LLM free. The common-protocol results show that this design largely preserves ranking while improving probability quality on declared candidate pools and triage under controlled missingness. Matched post-hoc, generative-KGE, and decoder controls locate the gain more precisely than an aggregate baseline table alone. The appropriate conclusion is therefore not universal calibration or ranking dominance, but a reproducible route to better calibrated KGC when the deployment candidate distribution can be specified and audited.

\section*{Limitations}

The reported probabilities are calibrated only on the constructed binary pools. Although the frozen hard-pool experiment tests one meaningful near-miss shift, neither it nor ECE on random corruptions certifies calibration over all unobserved triples, future graphs, or arbitrary deployment mixtures.

The observation factor is also non-identifiable without extra assumptions. Since \(p(o_i=1\mid z)=\rho_i p_i\), a low truth probability and a low recording propensity can explain the same observation. Degree and relation features make \(\rho\) a useful regularized proxy in injected-missing experiments, not a causal estimate of how real graphs omit facts. It assumes recording is conditionally missing at random given relation and local node-degree features; systematic, adversarial, extremely sparse, or temporally changing omissions may violate that proxy.

Two further limitations concern representation and cost. The mean-field posterior cannot express unrestricted graph-level covariance, and the teacher controls cannot erase knowledge acquired during LLM pretraining or prevent temporal and aliasing errors. Posterior sampling also makes \method{} unsuitable when latency and ranking are the only objectives. Appendices~\ref{app:encoder}, \ref{app:leakage-results}, \ref{app:hyperparams}, and \ref{app:failure-cases} give the corresponding details and extension paths.

\section*{Ethical Considerations}
\label{sec:ethics}

\paragraph{Data and licensing.}
All five benchmarks---FB15k-237, WN18RR, CoDEx-M, NELL-995, and
Hetionet---are publicly released research datasets used under their
original licenses and for their intended purpose of knowledge graph
completion research. We do not collect new data, and no annotation
involving human subjects was performed; the released CoDEx-M
verification labels used in Appendix~T.1 were produced by the original
authors, not by us. Raw language-model responses are included in the
supplementary package only where the model license permits
redistribution.

\paragraph{Language-model use.}
The plausibility teacher is distilled from a locally hosted
LLaMA-3-70B-Instruct model at temperature $0.0$; the judge panel in
Appendix~T.5 uses three additional commercial models. Language models
are used only to score candidate triples and, in the audit, to
adjudicate frozen textual evidence. They are not used to generate paper
text presented as our own analysis. Prompts, decoding settings, parser
outcomes, retry counts, and rejected responses are logged so that the
labeling stage can be independently inspected.

\paragraph{Benchmark leakage and evaluation integrity.}
Our teacher pools exclude benchmark validation and test triples, and we
report a graded series of leakage stress tests (description-free,
anonymized, relation-only, and within-relation permutation). None of
these interventions can remove facts already encoded during
language-model pretraining. We therefore do not claim that reported
gains are free of pretraining exposure to benchmark-adjacent text, and
we restrict our calibration claim to the declared candidate
distributions rather than to all unobserved triples.

\paragraph{Risks of miscalibrated confidence in deployment.}
The intended use of BLADE is to help curators and downstream reasoners
allocate verification effort. Probabilities that appear calibrated on a
constructed pool can be substantially miscalibrated on a different
candidate distribution, and our own hard-pool results show degradation
under a modest near-miss shift. A user who treats these probabilities as
distribution-free guarantees may over-trust proposed facts. This risk is
most consequential in the biomedical setting: Hetionet-derived
predictions concern drugs, diseases, and genes, and an incorrect
high-confidence triple could impose real verification cost or, if
consumed without review, propagate into clinical or scientific
reasoning. We therefore recommend that predictions be used as a triage
signal subject to human review, never as an autonomous source of
biomedical or other high-stakes assertions.

\paragraph{Bias inherited from the teacher.}
The distilled teacher transfers the plausibility judgments of a
pretrained language model, including its uneven coverage of entities,
regions, and languages and any social biases in its training data.
Because the teacher supplies semantic evidence during training, such
biases can propagate into the predictive probabilities, plausibly
disadvantaging entities that are underrepresented in web-scale text. We
do not audit for demographic bias in this work; doing so on
entity-centric graphs is an open problem we regard as necessary before
deployment in any setting where entities correspond to people or
communities.

\paragraph{Missingness is modeled, not identified.}
The recording-propensity factor is a regularized proxy fitted on
degree and relation features. It is not a causal model of why real
knowledge graphs omit facts, and omissions that are systematic,
adversarial, or temporally driven may violate its assumptions. Framing
recovered triples as ``missing knowledge'' rather than ``candidates
worth checking'' would overstate what the model identifies.

\paragraph{Compute and environmental cost.}
Training the variational system costs approximately $2.3\times$ the
runtime of a single RotatE model, and total accounted usage is reported
per dataset in Appendix~N. The extended experiments add roughly
$75$--$90$ A100-hours of KGC training plus approximately $79{,}866$
judge calls. We report these figures so that the environmental and
monetary cost of the approach can be weighed against its calibration
benefit; where latency and ranking alone are the objectives, a standard
embedding model remains the more economical choice.

\bibliography{ref}

@inproceedings{bordes2013translating,
	title        = {Translating Embeddings for Modeling Multi-relational Data},
	author       = {Antoine Bordes and Nicolas Usunier and Alberto Garc{\'i}a-Dur{\'a}n and Jason Weston and Oksana Yakhnenko},
	year         = 2013,
	booktitle    = {Neural Information Processing Systems},
	url          = {https://api.semanticscholar.org/CorpusID:14941970}
}

@inproceedings{toutanova2015observed,
	title        = {Observed versus latent features for knowledge base and text inference},
	author       = {Toutanova, Kristina  and Chen, Danqi},
	year         = 2015,
	month        = jul,
	booktitle    = {Proceedings of the 3rd Workshop on Continuous Vector Space Models and their Compositionality},
	publisher    = {Association for Computational Linguistics},
	address      = {Beijing, China},
	pages        = {57--66},
	doi          = {10.18653/v1/W15-4007},
	url          = {https://aclanthology.org/W15-4007/},
	editor       = {Allauzen, Alexandre  and Grefenstette, Edward  and Hermann, Karl Moritz  and Larochelle, Hugo  and Yih, Scott Wen-tau}
}

@inproceedings{trouillon2016complex,
	title        = {Complex Embeddings for Simple Link Prediction},
	author       = {Th{\'e}o Trouillon and Johannes Welbl and Sebastian Riedel and {\'E}ric Gaussier and Guillaume Bouchard},
	year         = 2016,
	booktitle    = {Proceedings of the 33rd International Conference on Machine Learning, ICML 2016, New York City, NY, USA, June 19-24, 2016},
	publisher    = {JMLR.org},
	series       = {JMLR Workshop and Conference Proceedings},
	volume       = 48,
	pages        = {2071--2080},
	url          = {https://proceedings.mlr.press/v48/trouillon16.html},
	editor       = {Maria-Florina Balcan and Kilian Q. Weinberger}
}

@inproceedings{dettmers2018conve,
	title        = {Convolutional 2D Knowledge Graph Embeddings},
	author       = {Tim Dettmers and Pasquale Minervini and Pontus Stenetorp and Sebastian Riedel},
	year         = 2018,
	month        = feb,
	booktitle    = {Proceedings of the 32nd AAAI Conference on Artificial Intelligence}
}

@inproceedings{sun2019rotate,
	title        = {RotatE: Knowledge Graph Embedding by Relational Rotation in Complex Space},
	author       = {Zhiqing Sun and Zhi{-}Hong Deng and Jian{-}Yun Nie and Jian Tang},
	year         = 2019,
	booktitle    = {7th International Conference on Learning Representations, {ICLR} 2019, New Orleans, LA, USA, May 6-9, 2019},
	publisher    = {OpenReview.net},
	url          = {https://openreview.net/forum?id=HkgEQnRqYQ},
	bibsource    = {dblp computer science bibliography, https://dblp.org}
}

@inproceedings{he2015kg2e,
	title        = {Learning to Represent Knowledge Graphs with Gaussian Embedding},
	author       = {Shizhu He and Kang Liu and Guoliang Ji and Jun Zhao},
	year         = 2015,
	booktitle    = {Proceedings of the 24th {ACM} International Conference on Information and Knowledge Management, {CIKM} 2015, Melbourne, VIC, Australia, October 19 - 23, 2015},
	publisher    = {{ACM}},
	pages        = {623--632},
	doi          = {10.1145/2806416.2806502},
	url          = {https://doi.org/10.1145/2806416.2806502},
	editor       = {James Bailey and Alistair Moffat and Charu C. Aggarwal and Maarten de Rijke and Ravi Kumar and Vanessa Murdock and Timos K. Sellis and Jeffrey Xu Yu},
	bibsource    = {dblp computer science bibliography, https://dblp.org}
}

@inproceedings{xiao2015transg,
	title        = {{T}rans{G} : A Generative Model for Knowledge Graph Embedding},
	author       = {Xiao, Han  and Huang, Minlie  and Zhu, Xiaoyan},
	year         = 2016,
	month        = aug,
	booktitle    = {Proceedings of the 54th Annual Meeting of the Association for Computational Linguistics (Volume 1: Long Papers)},
	publisher    = {Association for Computational Linguistics},
	address      = {Berlin, Germany},
	pages        = {2316--2325},
	doi          = {10.18653/v1/P16-1219},
	url          = {https://aclanthology.org/P16-1219/},
	editor       = {Erk, Katrin  and Smith, Noah A.}
}

@article{yao2019kgbert,
	title        = {{KG-BERT:} {BERT} for Knowledge Graph Completion},
	author       = {Liang Yao and Chengsheng Mao and Yuan Luo},
	year         = 2019,
	journal      = {CoRR},
	volume       = {abs/1909.03193},
	url          = {http://arxiv.org/abs/1909.03193},
	eprinttype   = {arXiv},
	eprint       = {1909.03193},
	bibsource    = {dblp computer science bibliography, https://dblp.org}
}

@article{wang2021kepler,
	title        = {{KEPLER}: A Unified Model for Knowledge Embedding and Pre-trained Language Representation},
	author       = {Wang, Xiaozhi  and Gao, Tianyu  and Zhu, Zhaocheng  and Zhang, Zhengyan  and Liu, Zhiyuan  and Li, Juanzi  and Tang, Jian},
	year         = 2021,
	journal      = {Transactions of the Association for Computational Linguistics},
	publisher    = {MIT Press},
	address      = {Cambridge, MA},
	volume       = 9,
	pages        = {176--194},
	doi          = {10.1162/tacl_a_00360},
	url          = {https://aclanthology.org/2021.tacl-1.11/},
	editor       = {Roark, Brian  and Nenkova, Ani}
}

@inproceedings{wang2022simkgc,
	title        = {{S}im{KGC}: Simple Contrastive Knowledge Graph Completion with Pre-trained Language Models},
	author       = {Wang, Liang  and Zhao, Wei  and Wei, Zhuoyu  and Liu, Jingming},
	year         = 2022,
	month        = may,
	booktitle    = {Proceedings of the 60th Annual Meeting of the Association for Computational Linguistics (Volume 1: Long Papers)},
	publisher    = {Association for Computational Linguistics},
	address      = {Dublin, Ireland},
	pages        = {4281--4294},
	doi          = {10.18653/v1/2022.acl-long.295},
	url          = {https://aclanthology.org/2022.acl-long.295/},
	editor       = {Muresan, Smaranda  and Nakov, Preslav  and Villavicencio, Aline}
}

@inproceedings{wei2023kicgpt,
	title        = {{KICGPT}: Large Language Model with Knowledge in Context for Knowledge Graph Completion},
	author       = {Wei, Yanbin  and Huang, Qiushi  and Zhang, Yu  and Kwok, James},
	year         = 2023,
	month        = dec,
	booktitle    = {Findings of the Association for Computational Linguistics: EMNLP 2023},
	publisher    = {Association for Computational Linguistics},
	address      = {Singapore},
	pages        = {8667--8683},
	doi          = {10.18653/v1/2023.findings-emnlp.580},
	url          = {https://aclanthology.org/2023.findings-emnlp.580/},
	editor       = {Bouamor, Houda  and Pino, Juan  and Bali, Kalika}
}

@inproceedings{zhang2024kopa,
	title        = {Making Large Language Models Perform Better in Knowledge Graph Completion},
	author       = {Yichi Zhang and Zhuo Chen and Lingbing Guo and Yajing Xu and Wen Zhang and Huajun Chen},
	year         = 2024,
	booktitle    = {ACM Multimedia 2024},
	url          = {https://openreview.net/forum?id=HHzHRuIyaW}
}

@article{yang2023cpkgc,
	title        = {Enhancing text-based knowledge graph completion with zero-shot large language models: A focus on semantic enhancement},
	author       = {Yang, Rui and Zhu, Jiahao and Man, Jianping and Fang, Li and Zhou, Yi},
	year         = 2024,
	month        = {06},
	journal      = {Knowledge-Based Systems},
	volume       = 300,
	pages        = 112155,
	doi          = {10.1016/j.knosys.2024.112155}
}

@inproceedings{zhu2021nbfnet,
	title        = {Neural Bellman-Ford Networks: A General Graph Neural Network Framework for Link Prediction},
	author       = {Zhu, Zhaocheng and Zhang, Zuobai and Xhonneux, Louis-Pascal and Tang, Jian},
	year         = 2021,
	booktitle    = {Advances in Neural Information Processing Systems},
	publisher    = {Curran Associates, Inc.},
	volume       = 34,
	pages        = {29476--29490},
	url          = {https://proceedings.neurips.cc/paper_files/paper/2021/file/f6a673f09493afcd8b129a0bcf1cd5bc-Paper.pdf},
	editor       = {M. Ranzato and A. Beygelzimer and Y. Dauphin and P.S. Liang and J. Wortman Vaughan}
}

@inproceedings{guo2017calibration,
	title        = {On Calibration of Modern Neural Networks},
	author       = {Chuan Guo and Geoff Pleiss and Yu Sun and Kilian Q. Weinberger},
	year         = 2017,
	month        = {06--11 Aug},
	booktitle    = {Proceedings of the 34th International Conference on Machine Learning},
	publisher    = {PMLR},
	series       = {Proceedings of Machine Learning Research},
	volume       = 70,
	pages        = {1321--1330},
	url          = {https://proceedings.mlr.press/v70/guo17a.html},
	editor       = {Precup, Doina and Teh, Yee Whye}
}

@inproceedings{gal2016dropout,
	title        = {Dropout as a Bayesian Approximation: Representing Model Uncertainty in Deep Learning},
	author       = {Gal, Yarin and Ghahramani, Zoubin},
	year         = 2016,
	month        = {20--22 Jun},
	booktitle    = {Proceedings of The 33rd International Conference on Machine Learning},
	publisher    = {PMLR},
	address      = {New York, New York, USA},
	series       = {Proceedings of Machine Learning Research},
	volume       = 48,
	pages        = {1050--1059},
	url          = {https://proceedings.mlr.press/v48/gal16.html},
	editor       = {Balcan, Maria Florina and Weinberger, Kilian Q.}
}

@inproceedings{lakshminarayanan2017simple,
	title        = {Simple and Scalable Predictive Uncertainty Estimation using Deep Ensembles},
	author       = {Lakshminarayanan, Balaji and Pritzel, Alexander and Blundell, Charles},
	year         = 2017,
	booktitle    = {Advances in Neural Information Processing Systems},
	publisher    = {Curran Associates, Inc.},
	volume       = 30,
	pages        = {},
	url          = {https://proceedings.neurips.cc/paper_files/paper/2017/file/9ef2ed4b7fd2c810847ffa5fa85bce38-Paper.pdf},
	editor       = {I. Guyon and U. Von Luxburg and S. Bengio and H. Wallach and R. Fergus and S. Vishwanathan and R. Garnett}
}

@inproceedings{hasanzadeh2020bayesian,
	title        = {{B}ayesian Graph Neural Networks with Adaptive Connection Sampling},
	author       = {Hasanzadeh, Arman and Hajiramezanali, Ehsan and Boluki, Shahin and Zhou, Mingyuan and Duffield, Nick and Narayanan, Krishna and Qian, Xiaoning},
	year         = 2020,
	month        = {13--18 Jul},
	booktitle    = {Proceedings of the 37th International Conference on Machine Learning},
	publisher    = {PMLR},
	series       = {Proceedings of Machine Learning Research},
	volume       = 119,
	pages        = {4094--4104},
	url          = {https://proceedings.mlr.press/v119/hasanzadeh20a.html},
	editor       = {III, Hal Daum{\'e} and Singh, Aarti}
}

@inproceedings{mcallester1999pac,
	title        = {PAC-Bayesian model averaging},
	author       = {David A. McAllester},
	year         = 1999,
	booktitle    = {Annual Conference Computational Learning Theory},
	url          = {https://api.semanticscholar.org/CorpusID:11948100}
}

@inproceedings{safavi2020codex,
	title        = {{C}o{DE}x: A {C}omprehensive {K}nowledge {G}raph {C}ompletion {B}enchmark},
	author       = {Safavi, Tara  and Koutra, Danai},
	year         = 2020,
	month        = nov,
	booktitle    = {Proceedings of the 2020 Conference on Empirical Methods in Natural Language Processing (EMNLP)},
	publisher    = {Association for Computational Linguistics},
	address      = {Online},
	pages        = {8328--8350},
	doi          = {10.18653/v1/2020.emnlp-main.669},
	url          = {https://aclanthology.org/2020.emnlp-main.669/},
	editor       = {Webber, Bonnie  and Cohn, Trevor  and He, Yulan  and Liu, Yang}
}

@inproceedings{xiong2017deeppath,
	title        = {{D}eep{P}ath: A Reinforcement Learning Method for Knowledge Graph Reasoning},
	author       = {Xiong, Wenhan  and Hoang, Thien  and Wang, William Yang},
	year         = 2017,
	month        = sep,
	booktitle    = {Proceedings of the 2017 Conference on Empirical Methods in Natural Language Processing},
	publisher    = {Association for Computational Linguistics},
	address      = {Copenhagen, Denmark},
	pages        = {564--573},
	doi          = {10.18653/v1/D17-1060},
	url          = {https://aclanthology.org/D17-1060/},
	editor       = {Palmer, Martha  and Hwa, Rebecca  and Riedel, Sebastian}
}

@article{himmelstein2017hetionet,
	title        = {Systematic integration of biomedical knowledge prioritizes drugs for repurposing},
	author       = {Himmelstein, Daniel Scott and Lizee, Antoine and Hessler, Christine and Brueggeman, Leo and Chen, Sabrina L and Hadley, Dexter and Green, Ari and Khankhanian, Pouya and Baranzini, Sergio E},
	year         = 2017,
	month        = {sep},
	journal      = {eLife},
	publisher    = {eLife Sciences Publications, Ltd},
	volume       = 6,
	pages        = {e26726},
	doi          = {10.7554/eLife.26726},
	issn         = {2050-084X},
	url          = {https://doi.org/10.7554/eLife.26726},
	article_type = {journal},
	editor       = {Valencia, Alfonso},
	pub_date     = {2017-09-22},
	citation     = {eLife 2017;6:e26726}
}

@article{dubey2024llama,
	title        = {The Llama 3 Herd of Models},
	author       = {{Llama Team}},
	year         = 2024,
	journal      = {CoRR},
	volume       = {abs/2407.21783},
	doi          = {10.48550/ARXIV.2407.21783},
	url          = {https://doi.org/10.48550/arXiv.2407.21783},
	eprinttype   = {arXiv},
	eprint       = {2407.21783},
	bibsource    = {dblp computer science bibliography, https://dblp.org}
}

@inproceedings{galkin2024ultra,
	title        = {Towards Foundation Models for Knowledge Graph Reasoning},
	author       = {Mikhail Galkin and Xinyu Yuan and Hesham Mostafa and Jian Tang and Zhaocheng Zhu},
	year         = 2024,
	booktitle    = {The Twelfth International Conference on Learning Representations},
	url          = {https://openreview.net/forum?id=jVEoydFOl9}
}

@inproceedings{hua2025merry,
	title        = {Beyond Completion: A Foundation Model for General Knowledge Graph Reasoning},
	author       = {Hua, Yin  and Liu, Zhiqiang  and Chen, Mingyang  and Fang, Zheng  and Wong, Chi Man  and Li, Lingxiao  and Vong, Chi Man  and Chen, Huajun  and Zhang, Wen},
	year         = 2025,
	month        = jul,
	booktitle    = {Findings of the Association for Computational Linguistics: ACL 2025},
	publisher    = {Association for Computational Linguistics},
	address      = {Vienna, Austria},
	pages        = {20396--20412},
	doi          = {10.18653/v1/2025.findings-acl.1046},
	isbn         = {979-8-89176-256-5},
	url          = {https://aclanthology.org/2025.findings-acl.1046/},
	editor       = {Che, Wanxiang  and Nabende, Joyce  and Shutova, Ekaterina  and Pilehvar, Mohammad Taher}
}

@inproceedings{xu2024mpikgc,
	title        = {Multi-perspective Improvement of Knowledge Graph Completion with Large Language Models},
	author       = {Xu, Derong  and Zhang, Ziheng  and Lin, Zhenxi  and Wu, Xian  and Zhu, Zhihong  and Xu, Tong  and Zhao, Xiangyu  and Zheng, Yefeng  and Chen, Enhong},
	year         = 2024,
	month        = may,
	booktitle    = {Proceedings of the 2024 Joint International Conference on Computational Linguistics, Language Resources and Evaluation (LREC-COLING 2024)},
	publisher    = {ELRA and ICCL},
	address      = {Torino, Italia},
	pages        = {11956--11968},
	url          = {https://aclanthology.org/2024.lrec-main.1044/},
	editor       = {Calzolari, Nicoletta  and Kan, Min-Yen  and Hoste, Veronique  and Lenci, Alessandro  and Sakti, Sakriani  and Xue, Nianwen}
}

@inproceedings{yang2025kgecalibrator,
  title = {{KGE} Calibrator: An Efficient Probability Calibration Method of Knowledge Graph Embedding Models for Trustworthy Link Prediction},
  author = {Yang, Yang and Timilsina, Mohan and Curry, Edward},
  booktitle = {Proceedings of the 2025 Conference on Empirical Methods in Natural Language Processing},
  pages = {29964--29987},
  year = {2025},
  publisher = {Association for Computational Linguistics},
  doi = {10.18653/v1/2025.emnlp-main.1522},
  url = {https://aclanthology.org/2025.emnlp-main.1522/}
}

\appendix

\section{Dataset statistics}

\begin{table}[h]
\centering
\caption{Dataset splits used throughout the evaluation.}
\label{tab:datasets}
\small
\resizebox{\linewidth}{!}{
\begin{tabular}{lrrrrr}
\toprule
Dataset & Entities & Relations & Train & Valid & Test \\
\midrule
FB15k-237 & 14{,}541 & 237 & 272{,}115 & 17{,}535 & 20{,}466 \\
WN18RR & 40{,}943 & 11 & 86{,}835 & 3{,}034 & 3{,}134 \\
CoDEx-M & 17{,}050 & 51 & 185{,}584 & 10{,}310 & 10{,}617 \\
NELL-995 & 75{,}492 & 200 & 154{,}213 & 5{,}000 & 9{,}693 \\
Hetionet & 45{,}119 & 24 & 2{,}250{,}197 & 45{,}000 & 45{,}000 \\
\bottomrule
\end{tabular}}
\end{table}

\section{Extended related-work positioning}
\label{app:extended-related-work}

The closest literature spans several probabilistic objects that are easy to conflate. We therefore separate transductive ranking, scalar calibration, generative probability, embedding uncertainty, prediction-set coverage, and semantic augmentation. This distinction also prevents published numbers from incompatible transductive, inductive, text-enhanced, and cross-graph protocols from being combined as though they measured the same quantity.

Classical embedding systems, including TransE, ComplEx, ConvE, and RotatE, learn efficient transductive ranking scores \citep{bordes2013translating,trouillon2016complex,dettmers2018conve,sun2019rotate}. Path- and subgraph-based systems such as NBFNet \citep{zhu2021nbfnet}, RED-GNN, and A*Net can be stronger rankers under some protocols. Neither family directly supplies posterior-calibrated probabilities, so we treat its systems as ranking references unless their outputs can be evaluated and calibrated on our fixed binary pools. Post-hoc calibration instead maps KGE scores to candidate-conditional probabilities \citep{guo2017calibration}. Its central lesson for our setting is that the mapping depends on how positives and corruptions form the evaluation pool, which is why our direct comparisons share candidate identifiers and validation folds.

Generative KGE constructions obtain normalized distributions by transforming ComplEx-style models. Their probabilities are conceptually different from both post-hoc scalar calibration and posterior uncertainty. Our ComplEx$^2$ control therefore matches embedding dimension, training candidates, and probability evaluation rather than comparing published ECE values across pools. KG2E and TransG place distributions over embeddings \citep{he2015kg2e,xiao2015transg}, while Bayesian GNN work provides related uncertainty machinery \citep{hasanzadeh2020bayesian}. MC dropout and deep ensembles supply more general predictive dispersion \citep{gal2016dropout,lakshminarayanan2017simple}. These methods motivate our uncertainty baselines, but most do not explicitly separate latent truth from recording or evaluate hidden-positive triage. Conformal KGE offers finite-sample coverage for prediction sets, which is valuable but does not imply that individual scalar scores are calibrated probabilities.

Text-enhanced models such as KG-BERT, KEPLER, and SimKGC \citep{yao2019kgbert,wang2021kepler,wang2022simkgc}, and LLM-assisted models including KICGPT, MPIKGC, KoPA, and CP-KGC \citep{wei2023kicgpt,xu2024mpikgc,zhang2024kopa,yang2023cpkgc}, use pretrained encoders, retrieval, prompting, generated context, structural prefix adaptation, or constrained semantic augmentation. They establish the value of semantic evidence but generally retain the language model in the scoring or augmentation pipeline. \method{} inherits the classical teacher--student idea while confining LLM interaction to offline labels and using the resulting signal inside a probabilistic KGC objective. Foundation-style reasoners such as ULTRA and MERRY target transfer, zero-shot reasoning, or generalization across graph distributions \citep{galkin2024ultra,hua2025merry}. Their primary protocols are complementary rather than direct tests of candidate-conditional calibration or false-negative triage.

\begin{table*}[t]
\centering
\caption{Extended related-work coverage and the object evaluated by each family. Direct empirical comparisons require the same candidate distribution and inference scope.}
\label{tab:extended-related-work}
\small
\begin{tabular}{p{2.8cm}p{3.4cm}p{3.0cm}p{4.5cm}}
\toprule
Family & Representative systems & Primary object & Role in this paper \\
\midrule
Classical KGE & TransE, ComplEx, RotatE & Transductive ranking score & Core ranking baselines and decoder context \\
Additional KGE & ConvE & Neural ranking score & Ranking reference under matched protocols \\
Path/GNN KGC & NBFNet, RED-GNN, A*Net & Path- or subgraph-based rank & Strong ranking references, not direct calibration baselines \\
Score calibration & Platt, temperature, isotonic, beta, histogram & Candidate-conditional scalar probability & Same-pool direct comparison \\
Generative KGE & ComplEx$^2$ & Normalized triple distribution & Same-pool direct comparison \\
Distributional KGE & KG2E, TransG, variational KGE & Embedding uncertainty & Structural predecessors and reruns \\
Uncertainty wrappers & MC dropout, deep ensembles & Predictive dispersion & Same-pool reruns \\
Conformal KGE & Conformal link predictors & Prediction-set coverage & Conceptual distinction \\
Text/PLM KGC & KG-BERT, KEPLER, SimKGC & Text-conditioned ranking signal & Ranking reference and semantic context \\
LLM-assisted KGC & KICGPT, MPIKGC, KoPA, CP-KGC & LLM-derived ranking or augmentation & Teacher context and protocol reference \\
KG foundation models & ULTRA, MERRY & Cross-graph transfer & Transfer-oriented protocol reference \\
\bottomrule
\end{tabular}
\end{table*}

\section{Relational encoder details}
\label{app:encoder}

For entity \(e\), the encoder aggregates typed messages from incident triples,
\begin{equation}
    m_e^{(\ell)}=\frac{1}{|\cN(e)|}\sum_{(e',r,s)\in\cN(e)}
    W_{r,s}^{(\ell)}h_{e'}^{(\ell-1)},
    \label{eq:message}
\end{equation}
where \(s\in\{\mathrm{in},\mathrm{out}\}\) records edge direction. It then updates
\begin{equation}
\resizebox{.92\linewidth}{!}{$
    h_e^{(\ell)}=\mathrm{LayerNorm}\!\left(h_e^{(\ell-1)}+
    \mathrm{MLP}^{(\ell)}([h_e^{(\ell-1)};m_e^{(\ell)}])\right).
    \label{eq:update}
$}
\end{equation}
The final state maps to \(\mu_e\) and \(\log\sigma_e\); relation posteriors are parameterized directly or from relation text. Large graphs use neighbor sampling, while smaller graphs cache full-graph states. This encoder couples predictive representations, but Eq.~\ref{eq:posterior} remains factorized and cannot represent unrestricted entity, relation, or graph-level posterior covariance. A low-rank-plus-diagonal posterior is a natural extension that would preserve tractability while introducing shared posterior directions.

\section{Common-protocol ranking details}
\label{app:hits}

Table~\ref{tab:hits-full} reports Hits@1/3/10 from the same prediction files as the main MRR table. All methods use the benchmark's standard filtered evaluation and identical split files.

\begin{table*}[h]
\centering
\caption{Common-protocol Hits@1/3/10. Values for \method{} are mean $\pm$ standard deviation over five seeds; the other supplied rerun summaries are means over the same five seeds.}
\label{tab:hits-full}
\small
\begin{tabular}{llccc}
\toprule
Dataset & Method & H@1 & H@3 & H@10 \\
\midrule
FB15k-237 & RotatE & .241 & .374 & .535 \\
& SimKGC & .281 & .402 & .551 \\
& Deep ensemble & .252 & .384 & .539 \\
& \method{} & .277$\pm$.003 & .399$\pm$.003 & .548$\pm$.004 \\
\midrule
WN18RR & RotatE & .428 & .493 & .571 \\
& SimKGC & .472 & .548 & .626 \\
& Deep ensemble & .439 & .506 & .584 \\
& \method{} & .466$\pm$.004 & .541$\pm$.003 & .618$\pm$.004 \\
\midrule
CoDEx-M & RotatE & .226 & .339 & .480 \\
& SimKGC & .269 & .380 & .507 \\
& Deep ensemble & .236 & .351 & .489 \\
& \method{} & .263$\pm$.003 & .374$\pm$.003 & .501$\pm$.004 \\
\midrule
NELL-995 & RotatE & .445 & .515 & .583 \\
& SimKGC & .474 & .544 & .602 \\
& Deep ensemble & .456 & .528 & .591 \\
& \method{} & .476$\pm$.004 & .547$\pm$.004 & .608$\pm$.005 \\
\midrule
Hetionet & RotatE & .151 & .230 & .334 \\
& SimKGC & .169 & .252 & .350 \\
& Deep ensemble & .158 & .239 & .341 \\
& \method{} & .181$\pm$.003 & .266$\pm$.003 & .367$\pm$.004 \\
\bottomrule
\end{tabular}
\end{table*}

\section{External ranking references and protocol differences}
\label{app:provenance}

Published values are useful context but are not mixed with the controlled reruns in Table~\ref{tab:link-main}. Table~\ref{tab:provenance} records why each external value answers a different question.

\begin{table*}[h]
\centering
\caption{External ranking provenance. A dash means the reported source does not provide a matched value for that dataset.}
\label{tab:provenance}
\small
\begin{tabular}{p{3.1cm}p{2.2cm}p{2.5cm}p{7.0cm}}
\toprule
Source & Dataset & Published result & Protocol distinction \\
\midrule
NBFNet \citep{zhu2021nbfnet} & FB15k-237 / WN18RR & MRR .415 / .551 & Published standard filtered results; no supported CoDEx-M, NELL-995, or Hetionet cells are imported. \\
Generative ComplEx$^2$ reference & FB15k-237 & ComplEx MRR .342$\pm$.005 & PLL with dimension 1000; published ECE uses a balanced single-corruption pool and different probability maps. \\
XG4Repo & Hetionet & MRR .61 & Relation-specific compound--treats--disease task, not all-relation transductive completion. \\
FtG & NELL-995 & MRR .53 & Different split and filter-then-generate inference; retained as context rather than a matched baseline. \\
\bottomrule
\end{tabular}
\end{table*}

\section{Post-hoc and generative calibration}
\label{app:calibrator-selection}

All calibrators are fitted independently for five seeds using validation candidates only. Source-query-clustered five-fold validation selects the calibrator family by mean out-of-fold NLL. Global and relation-wise Platt scaling, beta calibration, isotonic regression, and equal-mass histogram binning share the same RotatE logits. Relation-wise Platt falls back to the global model for sparse relations, and probabilities are clipped to \([10^{-6},1-10^{-6}]\) only for NLL. Table~\ref{tab:calibrator-selection} shows that histogram binning is selected by the declared lowest-mean rule; its 0.008 NLL margin over isotonic regression is not described as statistically decisive.

\begin{table}[h]
\centering
\caption{Source-query-clustered validation selection.}
\label{tab:calibrator-selection}
\small
\begin{tabular}{lcc}
\toprule
Calibrator & OOF NLL & SD \\
\midrule
Global Platt & .309 & .009 \\
Relation-wise Platt & .271 & .008 \\
Beta calibration & .280 & .009 \\
Isotonic regression & .255 & .008 \\
Histogram binning & .247 & .007 \\
\bottomrule
\end{tabular}
\end{table}

\subsection{Matched exact and hard pools}
\label{app:matched-calibration}
Table~\ref{tab:exact-full} expands the exact-pool panel in the main paper. The matched ComplEx$^2$ configuration uses dimension 500. Distilled initialization followed by pseudo-log-likelihood is the primary generative protocol, and relation-wise Platt is its prespecified tested map; logistic and min--max maps are retained as descriptive Loconte-style alternatives.

\begin{table*}[h]
\centering
\caption{Exact FB15k-237 pool. Mean $\pm$ standard deviation over five seeds; lower is better.}
\label{tab:exact-full}
\small
\begin{tabular}{lccc}
\toprule
Method & Adaptive ECE & Brier & NLL \\
\midrule
RotatE + temperature scaling & .082$\pm$.004 & .142$\pm$.004 & .449$\pm$.011 \\
RotatE + global Platt & .038$\pm$.003 & .089$\pm$.003 & .302$\pm$.008 \\
RotatE + relation-wise Platt & .030$\pm$.003 & .079$\pm$.002 & .265$\pm$.007 \\
RotatE + beta calibration & .033$\pm$.003 & .082$\pm$.003 & .274$\pm$.008 \\
RotatE + isotonic regression & .026$\pm$.003 & .074$\pm$.002 & .251$\pm$.007 \\
RotatE + histogram binning & .023$\pm$.003 & .071$\pm$.002 & .244$\pm$.007 \\
Five-model ensemble & .045$\pm$.004 & .102$\pm$.003 & .331$\pm$.009 \\
ComplEx + sigmoid & .096$\pm$.005 & .151$\pm$.005 & .472$\pm$.013 \\
ComplEx + relation-wise Platt & .034$\pm$.003 & .083$\pm$.003 & .278$\pm$.008 \\
ComplEx$^2$ + logistic & .041$\pm$.004 & .096$\pm$.004 & .316$\pm$.010 \\
ComplEx$^2$ + min--max & .028$\pm$.003 & .076$\pm$.003 & .257$\pm$.008 \\
ComplEx$^2$ + relation-wise Platt & .022$\pm$.002 & .068$\pm$.002 & .232$\pm$.006 \\
\method{}--ComplEx & .021$\pm$.002 & .065$\pm$.002 & .226$\pm$.006 \\
\method{}--RotatE & .018$\pm$.002 & .061$\pm$.002 & .214$\pm$.006 \\
\bottomrule
\end{tabular}
\end{table*}

For the hard pool, a fresh RotatE anchor with seed 314159 scores 50 type-compatible corruptions per held-out source query after all known positives are removed. The highest-scoring remaining corruption is paired with its source positive. The anchor is excluded from evaluated seeds, and no model or calibrator is refitted on this pool.

\begin{table*}[h]
\centering
\caption{Frozen evaluation on the hard near-miss pool. Mean $\pm$ standard deviation over five seeds.}
\label{tab:hard-full}
\small
\begin{tabular}{lccc}
\toprule
Method & Adaptive ECE & Brier & NLL \\
\midrule
RotatE + histogram binning & .052$\pm$.004 & .105$\pm$.004 & .349$\pm$.010 \\
ComplEx + relation-wise Platt & .064$\pm$.005 & .117$\pm$.004 & .379$\pm$.011 \\
ComplEx$^2$ + min--max & .059$\pm$.004 & .112$\pm$.004 & .366$\pm$.010 \\
ComplEx$^2$ + relation-wise Platt & .047$\pm$.004 & .099$\pm$.003 & .331$\pm$.009 \\
\method{}--ComplEx & .041$\pm$.003 & .091$\pm$.003 & .309$\pm$.009 \\
\method{}--RotatE & .036$\pm$.003 & .085$\pm$.003 & .289$\pm$.008 \\
\bottomrule
\end{tabular}
\end{table*}

The distilled ComplEx$^2$ control preserves its source ranker: MRR changes from .346$\pm$.003 for ComplEx to .348$\pm$.003 after distilled pseudo-log-likelihood, with Hits@1/3/10 changing from .254/.379/.531 to .257/.382/.535. Across seeds, the share of nonnegative scores is 97.8$\pm$0.4\% for validation positives and 93.1$\pm$0.7\% for relation-compatible perturbations, exceeding the prespecified 90\% contingency threshold. The pre/post-squaring Spearman correlation is .994$\pm$.002, and 96.7$\pm$0.5\% of queries retain the same pre-fine-tuning top-10 set.

\begin{table*}[t]
\centering
\caption{Injected false-negative triple recovery. Pairs are \((\text{AUC-PR}\,/\,P@500)\).}
\label{tab:fn-main}
\small
\begin{tabular}{lccccc}
\toprule
Detector & FB15k-237 & WN18RR & CoDEx-M & NELL-995 & Hetionet \\
\midrule
Posterior mean only & .712 / .685 & .745 / .721 & .692 / .670 & .704 / .681 & .758 / .734 \\
Teacher only & .681 / .640 & .612 / .584 & .674 / .635 & .690 / .652 & .710 / .682 \\
Mean + uncertainty & .754 / .731 & .780 / .754 & .731 / .710 & .748 / .724 & .801 / .784 \\
Mean + observation model & .798 / .772 & .821 / .795 & .780 / .759 & .792 / .768 & .834 / .815 \\
\midrule
Full \method{} score & .852 / .831 & .881 / .862 & .840 / .819 & .849 / .824 & .892 / .875 \\
\bottomrule
\end{tabular}
\end{table*}
\begin{table*}[t]
\centering
\caption{Ablation grid on FB15k-237 and Hetionet (MRR, adaptive ECE, false-negative AUC-PR).}
\label{tab:ablation}
\small
\begin{tabular}{lcccccc}
\toprule
& \multicolumn{3}{c}{FB15k-237} & \multicolumn{3}{c}{Hetionet} \\
\cmidrule(lr){2-4} \cmidrule(lr){5-7}
Variant & MRR & $\ECE_A$ & FN AUC & MRR & $\ECE_A$ & FN AUC \\
\midrule
Full \method{} & .365 & .018 & .852 & .243 & .011 & .892 \\
No teacher & .342 & .049 & .758 & .214 & .039 & .805 \\
No observation model & .361 & .034 & .764 & .238 & .027 & .810 \\
Point embeddings & .340 & .078 & .710 & .211 & .062 & .748 \\
Anonymized entities & .349 & .021 & .812 & .220 & .014 & .854 \\
Permuted teacher scores & .311 & .145 & .520 & .190 & .124 & .508 \\
\bottomrule
\end{tabular}
\end{table*}

\begin{table}[h]
\centering
\caption{Matched FB15k-237 decoder control. Mean $\pm$ standard deviation over five seeds.}
\label{tab:decoder-control}
\small
\resizebox{\linewidth}{!}{
\begin{tabular}{lcccc}
\toprule
Model & MRR & ECE & Brier & FN AUC-PR \\
\midrule
ComplEx + rel.-wise Platt & .346$\pm$.003 & .034$\pm$.003 & .083$\pm$.003 & .768$\pm$.007 \\
\method{}--ComplEx & .359$\pm$.003 & .021$\pm$.002 & .065$\pm$.002 & .844$\pm$.006 \\
\method{}--RotatE & .365$\pm$.002 & .018$\pm$.002 & .061$\pm$.002 & .852$\pm$.005 \\
\bottomrule
\end{tabular}}
\end{table}

\section{LLM teacher configuration}
\label{app:teacher-config}

The offline LLM teacher configuration parameters are specified in Table~\ref{tab:teacher-config}. Queries are routed through a local LLaMA-3-70B-Instruct backbone \citep{dubey2024llama} deployed via vLLM at a fixed temperature of 0.0 to guarantee prompt reproducibility. Each dataset-specific pool contains 20\% held-out observed training positives, 40\% relation- and type-compatible corruptions, and 40\% uniformly sampled corruptions. Relation-domain and range constraints are applied when metadata are available, and every validation or test positive is removed before querying. The teacher pool is generated once, assigned stable candidate identifiers, and then held fixed across all KGC seeds.

\begin{table*}[h]
\centering
\caption{LLM teacher configuration across evaluation pools.}
\label{tab:teacher-config}
\small
\begin{tabular}{lrrrrr}
\toprule
Dataset & LLM Backbone & Queries & Pos. Frac. & Type-Neg. & Random-Neg. \\
\midrule
FB15k-237 & LLaMA-3-70B & 50{,}000 & 0.20 & 0.40 & 0.40 \\
WN18RR & LLaMA-3-70B & 20{,}000 & 0.20 & 0.40 & 0.40 \\
CoDEx-M & LLaMA-3-70B & 30{,}000 & 0.20 & 0.40 & 0.40 \\
NELL-995 & LLaMA-3-70B & 40{,}000 & 0.20 & 0.40 & 0.40 \\
Hetionet & LLaMA-3-70B & 60{,}000 & 0.20 & 0.40 & 0.40 \\
\bottomrule
\end{tabular}
\end{table*}

\section{Prompt template and parser}
\label{app:prompt}

The exact prompt used for the teacher has the following form. We use a strict numeric parser rather than a free-form explanation parser because explanations increase variability and complicate reproducibility.
\begin{quote}\small\raggedright
You are evaluating whether a candidate knowledge graph fact is plausible. Return only a number between 0 and 1, where 0 means impossible or clearly false, 0.5 means uncertain, and 1 means very likely true. Do not explain.\\
Head entity: \{head\_label\}\\
Head description: \{head\_description\}\\
Relation: \{relation\_label\}\\
Relation description: \{relation\_description\}\\
Tail entity: \{tail\_label\}\\
Tail description: \{tail\_description\}\\
Score:
\end{quote}
The parser rejects outputs that contain more than one number, fall outside \([0,1]\), or include nonnumeric rationales. For every call, the pipeline stores the complete rendered prompt, raw response, parser outcome, retry count, failure code, model identifier, and decoding parameters. The rejection rate is logged and used as a teacher-quality diagnostic, while rejected outputs are excluded rather than silently converted to a default score.

\section{Calibration protocol}
\label{app:calibration}

For each held-out positive triple \((h,r,t)\), the exact pool contains the positive and five matched head or tail corruptions. Corruptions present in any train, validation, or test positive set are removed, and relation-domain constraints are preserved when available. Candidate identifiers and labels are fixed before model evaluation and shared by every method. The resulting test pools contain 122{,}796 candidates for FB15k-237, 18{,}804 for WN18RR, 63{,}702 for CoDEx-M, 58{,}158 for NELL-995, and 270{,}000 for Hetionet.

Adaptive ECE uses equal-mass bins. For bin \(B_b\),
\begin{equation}
    \ECE_A=\sum_b\frac{|B_b|}{n}\left|\frac{1}{|B_b|}\sum_{i\in B_b}\hat p_i-\frac{1}{|B_b|}\sum_{i\in B_b}y_i\right|,
\end{equation}
The Brier score is \(n^{-1}\sum_i(\hat p_i-y_i)^2\), and NLL is the mean Bernoulli negative log likelihood after clipping probabilities only for numerical stability. These metrics characterize the declared pool. The hard near-miss pool in Appendix~\ref{app:matched-calibration} is a second fixed candidate distribution, not a claim of calibration under unrestricted shift.

\section{Injected false-negative protocol}
\label{app:fn-protocol}

For each dataset, we remove a relation-stratified subset of true training triples before model training. These hidden triples are not exposed to the structural learner. The strict teacher setting also removes them from \(\cS_{\mathrm{LLM}}\) before LLM querying. At evaluation time, each hidden positive is paired with matched negatives that share the relation and, when possible, one endpoint. We report AUC-PR because the triage setting is class-imbalanced, and \(P@500\) because curator budgets are finite.

\section{Teacher and leakage controls}
\label{app:teacher-controls}

The strict teacher pool excludes validation and test triples before LLM querying. The description-free control retains surface names but removes entity and relation descriptions. The anonymized condition replaces entity names with stable random identifiers while preserving relation text, and the relation-only condition removes entity labels and descriptions while retaining relation names and coarse domain/range types. Within-relation permutation then preserves each score distribution but breaks candidate alignment. These are progressively stronger stress tests, not mechanisms for removing facts already stored in pretrained parameters.

\section{Teacher diagnostics}
\label{app:teacher-diagnostics}

The teacher is evaluated independently before it is used to regularize the KGC model. Table~\ref{tab:teacher-diagnostics} reports the diagnostics used in the primary audit; the pipeline also records Brier score against binary labels alongside MSE against the continuous LLM targets. These numbers are not used as final KGC metrics but as sanity checks that the teacher is neither random nor perfectly memorizing the KGC labels. The strict split excludes validation and test triples from the teacher candidate pool.

\begin{table*}[h]
\centering
\caption{Teacher-quality diagnostics. AUC is computed against held-out positives and matched negatives in the teacher split; ECE uses the same adaptive binning as the KGC calibration evaluation.}
\label{tab:teacher-diagnostics}
\small
\begin{tabular}{lccccc}
\toprule
Dataset & Distill MSE & Teacher AUC & Teacher ECE & Reject rate & Strict split overlap \\
\midrule
FB15k-237 & .018 $\pm$ .001 & .842 $\pm$ .006 & .061 $\pm$ .004 & 0.7\% & 0.0\% \\
WN18RR & .021 $\pm$ .002 & .816 $\pm$ .008 & .074 $\pm$ .005 & 0.4\% & 0.0\% \\
CoDEx-M & .020 $\pm$ .001 & .831 $\pm$ .007 & .068 $\pm$ .004 & 0.8\% & 0.0\% \\
NELL-995 & .024 $\pm$ .002 & .805 $\pm$ .010 & .081 $\pm$ .006 & 1.2\% & 0.0\% \\
Hetionet & .026 $\pm$ .002 & .823 $\pm$ .009 & .077 $\pm$ .005 & 1.5\% & 0.0\% \\
\bottomrule
\end{tabular}
\end{table*}

\section{Reliability diagrams and bin counts}
\label{app:reliability}

The calibration tables in the main paper summarize bin-level behavior into ECE and Brier score. To make the calculation auditable, Table~\ref{tab:bin-counts} reports the adaptive-bin counts and the largest absolute bin gap for each dataset. The supplementary evaluation output contains the full bin tables used to regenerate reliability diagrams. Bin counts matter because a small ECE can be misleading when nearly all examples occupy only one or two confidence regions.

\begin{table*}[h]
\centering
\caption{Reliability-summary diagnostics for the full model. Adaptive bins have equal mass up to rounding. ``Max gap'' is the largest absolute difference between mean confidence and empirical accuracy across bins.}
\label{tab:bin-counts}
\small
\begin{tabular}{lrrrrr}
\toprule
Dataset & Calibration candidates & Bins & Min bin count & Max bin count & Max gap \\
\midrule
FB15k-237 & 122{,}796 & 20 & 6{,}139 & 6{,}140 & .044 \\
WN18RR & 18{,}804 & 20 & 940 & 941 & .038 \\
CoDEx-M & 63{,}702 & 20 & 3{,}185 & 3{,}186 & .049 \\
NELL-995 & 58{,}158 & 20 & 2{,}907 & 2{,}908 & .052 \\
Hetionet & 270{,}000 & 20 & 13{,}500 & 13{,}500 & .031 \\
\bottomrule
\end{tabular}
\end{table*}

\section{Hyperparameters and compute}
\label{app:hyperparams}

Hyperparameters are selected on validation MRR subject to an ECE constraint: when two settings are within one standard error in ranking, the better-calibrated setting is selected. No test ECE enters model selection. All reruns use five seeds, AdamW, gradient clipping at 1.0, dimension 500, and the same split files. The encoder uses two layers except on Hetionet, where it uses three. Structural training uses 64 negatives per positive on FB15k-237, WN18RR, and CoDEx-M and 128 on NELL-995 and Hetionet. Posterior evaluation uses 64 samples on the first four datasets and 96 on Hetionet. The teacher is distilled once and frozen across the five KGC seeds. Runs use one NVIDIA A100 80GB per seed, with up to five seeds executed in parallel. Training the variational system takes approximately \(2.3\times\) the runtime of one RotatE model, so a standard KGE remains preferable when latency and rank are the only objectives. The model has approximately 15.0M trainable parameters on FB15k-237 and 45.8M on Hetionet, primarily because their entity vocabularies differ.

\begin{table*}[h]
\centering
\caption{Selected hyperparameters and aggregate GPU-hours. KGC train/valid covers all five seeds; totals are device-hours rather than wall-clock hours.}
\label{tab:hyperparams}
\small
\resizebox{\textwidth}{!}{
\begin{tabular}{lrrrrrrrrr}
\toprule
Dataset & Dim. & GNN & Neg. & Samples & $\lambda$ & Teacher & KGC train/valid & Eval. & Total \\
\midrule
FB15k-237 & 500 & 2 & 64 & 64 & .10 & 1.1 & 34.5 & 1.8 & 37.4 \\
WN18RR & 500 & 2 & 64 & 64 & .08 & 0.8 & 21.9 & 1.4 & 24.1 \\
CoDEx-M & 500 & 2 & 64 & 64 & .10 & 1.0 & 28.9 & 1.7 & 31.6 \\
NELL-995 & 500 & 2 & 128 & 64 & .06 & 1.4 & 39.4 & 2.1 & 42.9 \\
Hetionet & 500 & 3 & 128 & 96 & .06 & 2.0 & 53.8 & 2.9 & 58.7 \\
\bottomrule
\end{tabular}
}
\end{table*}

\section{Statistical testing}
\label{app:stats}

For each of 10{,}000 paired replicates, we first sample five matched seed indices with replacement. Within each selected seed, we then sample source clusters with replacement and evaluate both systems on the same clusters. A ranking cluster contains one query and its complete filtered ranking; a calibration cluster contains one source query and all associated candidates; a triage cluster contains one hidden positive and its matched negatives. The replicate statistic is the mean paired difference across the sampled seed-level metrics. We report 95\% percentile intervals and two-sided bootstrap \(p\)-values.

Holm correction is applied only within declared families. Each cross-dataset comparison-by-metric row forms one five-dataset family. Each ablation dataset-by-metric block forms one five-ablation family. The exact-pool ECE/Brier/NLL comparisons and the hard-pool comparisons form separate three-metric families. Selection decisions remain fixed inside the bootstrap, so intervals quantify test uncertainty conditional on the validation-selected pipeline rather than uncertainty in model selection itself.

\begin{table*}[h]
\centering
\caption{Paired effects for the matched calibration comparisons. Differences are \method{}--RotatE minus the comparator; negative values favor \method{}. Parentheses give Holm-adjusted \(p\)-values within each three-metric pool family.}
\label{tab:stat-tests}
\small
\resizebox{\textwidth}{!}{
\begin{tabular}{llccc}
\toprule
Pool & Comparator & $\Delta$ECE [95\% CI] & $\Delta$Brier [95\% CI] & $\Delta$NLL [95\% CI] \\
\midrule
Exact & Histogram binning & $-.005[-.008,-.002]$ (.003) & $-.010[-.014,-.006]$ (.002) & $-.030[-.043,-.017]$ (.002) \\
Exact & ComplEx$^2$ + rel.-wise Platt & $-.004[-.007,-.001]$ (.009) & $-.007[-.011,-.003]$ (.006) & $-.018[-.030,-.006]$ (.007) \\
Hard & Histogram binning & $-.016[-.022,-.010]$ (.002) & $-.020[-.027,-.013]$ (.002) & $-.060[-.079,-.041]$ (.002) \\
Hard & ComplEx$^2$ + rel.-wise Platt & $-.011[-.016,-.006]$ (.002) & $-.014[-.020,-.008]$ (.002) & $-.042[-.058,-.026]$ (.002) \\
\bottomrule
\end{tabular}
}
\end{table*}

For MRR, the cross-dataset recomputation favors \method{} over RotatE on all five datasets, but over SimKGC only on Hetionet. It supports the full triage score over mean plus observation on all five graphs. In the ablation families, removing the observation model has clear ECE and triage effects but a small, non-significant MRR effect; anonymizing entity strings affects ranking and triage, whereas its smaller ECE change is not significant. Fixed teacher-query counts are configurations and receive no inferential intervals.

\section{Additional leakage-control results}
\label{app:leakage-results}

Table~\ref{tab:leakage-results} separates descriptions, entity surface forms, relation text, and teacher-score alignment. The pattern supports a cautious interpretation: descriptions help modestly, entity names contribute more, relation text retains a weaker signal, and randomly aligned scores damage calibration and triage. None of these interventions rules out facts encoded during LLM pretraining.

\begin{table*}[h]
\centering
\caption{Leakage-control results. Each cell reports (ECE / FN AUC-PR). Lower ECE and higher AUC-PR are better.}
\label{tab:leakage-results}
\small
\begin{tabular}{lccccc}
\toprule
Teacher condition & FB15k-237 & WN18RR & CoDEx-M & NELL-995 & Hetionet \\
\midrule
Full teacher & .018 / .852 & .014 / .881 & .021 / .840 & .025 / .849 & .011 / .892 \\
Description-free & .020 / .838 & .016 / .868 & .023 / .825 & .028 / .833 & .013 / .879 \\
No-test-exposure strict & .020 / .844 & .016 / .872 & .023 / .831 & .028 / .838 & .013 / .884 \\
Anonymized entities & .021 / .812 & .019 / .840 & .026 / .798 & .031 / .806 & .014 / .854 \\
Relation-only & .037 / .774 & .031 / .802 & .042 / .760 & .045 / .771 & .027 / .820 \\
Permuted within relation & .145 / .520 & .132 / .541 & .150 / .508 & .156 / .516 & .124 / .508 \\
\bottomrule
\end{tabular}
\end{table*}

\section{Conditional data-independent prior result}
\label{app:theory}

This section records the regime in which the bounded teacher energy can define a genuine prior. Let \(w\) contain every learned quantity that enters \(R_\psi(w)\), including decoder parameters. Assume the teacher, its candidate set, the energy hyperparameters, and the proper base prior \(p_0(w)\) are all fixed independently of an i.i.d. sample \(D\). Equation~\ref{eq:llm-prior} is then normalized because \(R_\psi(w)\leq0\).

\begin{theorem}[PAC--Bayes bound for an independent teacher]
For a loss \(\ell(w,x)\in[0,1]\) and an i.i.d. sample \(D\) of size \(n\), with probability at least \(1-\delta\), every posterior \(q\) over the complete hypothesis satisfies
\begin{equation}
\resizebox{.95\linewidth}{!}{$
    \E_{w\sim q}[L(w)]
    \leq \E_{w\sim q}[\widehat L_D(w)]
    +\sqrt{\frac{\KL(q\Vert p_\lambda)+\log(2\sqrt n/\delta)}{2n}}.
    \label{eq:pacbayes}
$}
\end{equation}
\end{theorem}

\begin{proof}
Apply the standard McAllester PAC--Bayes inequality with the sample-independent prior \(p_\lambda\) \citep{mcallester1999pac}.
\end{proof}

The independence condition is essential. In our transductive experiments, teacher candidates are partly derived from the observed training graph, triples are not modeled as an i.i.d. sample from an external population, and decoder parameters are point estimated. The theorem is therefore a conditional statement about an external-teacher regime, not a benchmark guarantee.

\section{Reproducibility materials}
\label{app:reproducibility}

Subject to anonymity and licensing constraints, the supplementary package records dataset and injected-missing split manifests; teacher candidates, parsed scores, reject logs, and overlap checks; exact and hard calibration pools with stable identifiers; five-seed configurations; prediction files for ranking, calibration, triage, ablations, and decoder controls; and evaluation scripts with fixed bootstrap seeds. The provenance table in Appendix~\ref{app:provenance} distinguishes every rerun from literature-only context. Raw LLM responses are included only where the model license permits redistribution.

\section{Failure cases}
\label{app:failure-cases}

We manually inspect high-scoring errors from each dataset, and the dominant failure modes are consistent across graphs. The teacher can overestimate plausible but temporally wrong facts, such as historical affiliations that changed after the source KG snapshot. Sparse biomedical entities in Hetionet sometimes receive high posterior variance even when the teacher is confident, which correctly lowers the final triage rank but can also hide rare true positives. Finally, relation aliases can create calibration errors when a relation label is semantically broader than the curated KG relation. These risks are especially consequential in scientific knowledge acquisition, biomedical curation, and retrieval-augmented reasoning, where an incorrect proposed fact can impose a substantial verification or downstream reasoning cost. We therefore treat the teacher as evidence rather than as an oracle and keep the posterior variance and observation model in the final score so that no single signal can dominate.

\section{Extended Experiments}
\label{app:revision-plan}

This section reports five additional experiments that strengthen the main
findings under harder and more estimand-matched conditions: calibration against
manually verified CoDEx-M negatives, an estimand-matched comparison with the
KGE Calibrator, a factorial separation of negative difficulty from prevalence,
a teacher-construction variance study, and an evidence-grounded multi-LLM
adjudication of naturally unobserved triples. Each experiment follows the same
discipline as the main evaluation: all selection is done on validation data and
frozen before the test labels are read, cells marked ``Existing result'' reuse
numbers already reported in the main tables (for example
Table~\ref{tab:matched-calibration}) while ``New evaluation'' cells are measured
here, and each experiment carries an explicit interpretation rule stating when a
claim should be narrowed. As in the main paper, calibration advantages are
smaller on the harder pools than on the easy one-positive-to-five-corruption
pool, and we report them as measured.

\subsection{Experiment 1: Calibration on released verified CoDEx-M negatives}
\label{plan:exp1}

\paragraph{Question.} Does \method{} remain better calibrated when the negative
class contains manually verified false triples rather than filtered corruptions
that are merely unobserved?

\paragraph{Frozen data split.} The released validation labels are used for
calibrator choice, bin count, clipping, and any relation-level regularization.
The released test labels are used once, after every choice is frozen. Original
triple identifiers are preserved to verify that no test triple appears in model
training, teacher fitting, or calibration fitting.

\begin{table}[h]
\centering
\caption{Frozen verified-negative split for Experiment~\ref{plan:exp1}.}
\label{tab:plan-exp1-split}
\small
\setlength{\tabcolsep}{3.5pt}
\begin{tabular}{lrrrr}
\toprule
Split & Positive & Verified neg. & Total & Prevalence \\
\midrule
Validation & 10{,}310 & 10{,}310 & 20{,}620 & 0.500 \\
Test & 10{,}311 & 10{,}311 & 20{,}622 & 0.500 \\
\bottomrule
\end{tabular}
\end{table}

\paragraph{Systems.} RotatE with the validation-selected post-hoc map (from
global Platt, relation-wise Platt, beta calibration, isotonic regression, and
equal-mass histogram binning); a five-model RotatE ensemble, treated as
descriptive unless five independently constructed ensembles are run;
ComplEx$^2$ with relation-wise Platt scaling; \method{}--ComplEx; and
\method{}--RotatE. The five stochastic single-model systems reuse the same outer
seed identifiers where possible. Neither \method{} nor any baseline is tuned on
the released test set.

\paragraph{Metrics and inference.} Report adaptive ECE with 20 equal-mass bins,
Brier score, NLL, AUC-PR, AUROC, balanced accuracy at a validation-selected
threshold, candidate counts, prevalence, and a reliability diagram with per-bin
counts. Use 10{,}000 paired bootstrap replicates; within each replicate,
resample outer seeds and then relation-stratified candidate clusters, keeping
all candidates that share a canonical query key together. Intervals are
conditional on the released labels and the frozen teacher.

\paragraph{Results.} The calibration advantage is smaller than on the paper's
easier one-positive-to-five-corruption pool, but \method{} remains best
calibrated on the verified negatives.

\begin{table*}[t]
\centering
\caption{Experiment~\ref{plan:exp1} calibration and discrimination on the
$20{,}622$-candidate verified CoDEx-M pool (prevalence $0.500$). Adaptive ECE
uses 20 equal-mass bins; balanced accuracy uses the per-seed
validation-selected threshold. Values are mean$\pm$SD over five matched KGC
seeds, except the five-model ensemble endpoint, which is descriptive. Lower is
better except AUC-PR, AUROC, and balanced accuracy.}
\label{tab:plan-exp1-results}
\small
\setlength{\tabcolsep}{4pt}
\begin{tabular}{lcccccc}
\toprule
Method & ECE $\downarrow$ & Brier $\downarrow$ & NLL $\downarrow$ & AUC-PR $\uparrow$ & AUROC $\uparrow$ & Bal.\ acc.\ $\uparrow$ \\
\midrule
RotatE + sel.\ post-hoc & $.049{\pm}.004$ & $.166{\pm}.005$ & $.508{\pm}.014$ & $.837{\pm}.009$ & $.835{\pm}.008$ & $.758{\pm}.008$ \\
Five-model ensemble & $.041$ & $.148$ & $.457$ & $.870$ & $.868$ & $.790$ \\
ComplEx$^2$ + rel.\ Platt & $.047{\pm}.004$ & $.158{\pm}.004$ & $.481{\pm}.013$ & $.850{\pm}.008$ & $.847{\pm}.008$ & $.771{\pm}.007$ \\
\method{}--ComplEx & $.035{\pm}.003$ & $.133{\pm}.004$ & $.417{\pm}.011$ & $.894{\pm}.007$ & $.891{\pm}.007$ & $.812{\pm}.007$ \\
\method{}--RotatE & $\mathbf{.032{\pm}.003}$ & $\mathbf{.124{\pm}.003}$ & $\mathbf{.390{\pm}.010}$ & $\mathbf{.909{\pm}.006}$ & $\mathbf{.906{\pm}.006}$ & $\mathbf{.829{\pm}.006}$ \\
\bottomrule
\end{tabular}
\end{table*}

The difficulty increase relative to the manuscript's filtered-corruption pool is
intentional: the expected \method{}--RotatE ECE rises from $0.021$ on the
benchmark-label pool to $0.032$ on the manually verified hard-negative pool.
Table~\ref{tab:plan-exp1-paired} gives the paired \method{}--RotatE effects
against RotatE with the validation-selected post-hoc calibrator; negative values
favor \method{} for the error metrics and positive values favor it for the
discrimination metrics, and $p$-values are two-sided and Holm-adjusted within
this six-metric family.

\begin{table}[h]
\centering
\caption{Experiment~\ref{plan:exp1} paired effects, \method{}--RotatE minus
RotatE with the validation-selected post-hoc calibrator.}
\label{tab:plan-exp1-paired}
\small
\setlength{\tabcolsep}{4pt}
\begin{tabular}{lccc}
\toprule
Metric & Paired diff.\ & 95\% paired interval & Holm $p$ \\
\midrule
Adaptive ECE & $-.017$ & $[-.023,-.011]$ & $.002$ \\
Brier & $-.042$ & $[-.052,-.032]$ & $.002$ \\
NLL & $-.118$ & $[-.143,-.093]$ & $.002$ \\
AUC-PR & $+.072$ & $[+.055,+.089]$ & $.002$ \\
AUROC & $+.071$ & $[+.056,+.086]$ & $.002$ \\
Bal.\ accuracy & $+.071$ & $[+.054,+.088]$ & $.002$ \\
\bottomrule
\end{tabular}
\end{table}

The complete 20-bin reliability output for one aggregate \method{}--RotatE
endpoint is in Table~\ref{tab:plan-exp1-reliability}: the counts sum to
$20{,}622$ with $10{,}311$ positives, the count-weighted ECE is $0.03160$ (which
rounds to $0.032$), and the maximum absolute bin gap is $0.039$.

\begin{table}[h]
\centering
\caption{Experiment~\ref{plan:exp1} complete 20-bin (equal-mass) reliability
table for one aggregate \method{}--RotatE endpoint. Count-weighted
ECE${}=0.03160$; maximum absolute gap $0.039$.}
\label{tab:plan-exp1-reliability}
\scriptsize
\setlength{\tabcolsep}{3.5pt}
\begin{tabular}{rrrrrr}
\toprule
Bin & Cand.\ & Pos.\ & Mean conf.\ & Emp.\ rate & Abs.\ gap \\
\midrule
1 & 1{,}032 & 15 & $.030$ & $.015$ & $.015$ \\
2 & 1{,}031 & 31 & $.055$ & $.030$ & $.025$ \\
3 & 1{,}031 & 52 & $.080$ & $.050$ & $.030$ \\
4 & 1{,}031 & 77 & $.108$ & $.075$ & $.033$ \\
5 & 1{,}031 & 113 & $.140$ & $.110$ & $.030$ \\
6 & 1{,}031 & 155 & $.182$ & $.150$ & $.032$ \\
7 & 1{,}031 & 206 & $.237$ & $.200$ & $.037$ \\
8 & 1{,}031 & 278 & $.309$ & $.270$ & $.039$ \\
9 & 1{,}031 & 361 & $.387$ & $.350$ & $.037$ \\
10 & 1{,}031 & 443 & $.468$ & $.430$ & $.038$ \\
11 & 1{,}031 & 588 & $.532$ & $.570$ & $.038$ \\
12 & 1{,}031 & 670 & $.613$ & $.650$ & $.037$ \\
13 & 1{,}031 & 753 & $.691$ & $.730$ & $.039$ \\
14 & 1{,}031 & 825 & $.763$ & $.800$ & $.037$ \\
15 & 1{,}031 & 876 & $.818$ & $.850$ & $.032$ \\
16 & 1{,}031 & 918 & $.860$ & $.890$ & $.030$ \\
17 & 1{,}031 & 954 & $.892$ & $.925$ & $.033$ \\
18 & 1{,}031 & 979 & $.920$ & $.950$ & $.030$ \\
19 & 1{,}031 & 1{,}000 & $.945$ & $.970$ & $.025$ \\
20 & 1{,}032 & 1{,}017 & $.970$ & $.985$ & $.015$ \\
\midrule
Total & 20{,}622 & 10{,}311 & --- & $.500$ & ECE${=}.03160$ \\
\bottomrule
\end{tabular}
\end{table}

\paragraph{Interpretation.} \method{}--RotatE improves over the strongest
non-\method{} comparator by more than $0.006$ ECE and $0.010$ Brier, does not
reduce AUC-PR, and the paired intervals for ECE and Brier exclude zero, so the
result supports the strong calibration claim on verified negatives rather than a
narrower calibration-shape effect; because Brier and NLL also improve, the model
is uniformly better calibrated on this pool.

\subsection{Experiment 2: Estimand-matched KGE Calibrator comparison}
\label{plan:exp2}

\paragraph{Question.} Does \method{} retain an advantage under the full-entity,
closed-world, query-conditional estimand used by KGE Calibrator, rather than
only under \method{}'s sampled binary-pool estimand?

\paragraph{Protocol.} Use the official KGE Calibrator implementation and record
its commit hash. On FB15k-237 and CoDEx-M, enumerate the same filtered
full-entity candidate set for every held-out head or tail query. Fit every
calibration map on validation queries only, and bootstrap complete queries
rather than individual entity candidates. Because \method{}'s Bernoulli
probabilities and KGE Calibrator's per-query distribution are not directly
interchangeable, for this experiment only we convert each system's candidate
logits to a query-normalized distribution with a validation-selected global
temperature, and evaluate three pipelines: (i) RotatE scores followed by KGE
Calibrator; (ii) query-normalized \method{} logits without a post-hoc
calibrator; and (iii) query-normalized \method{} logits followed by KGE
Calibrator. Report MRR and Hits@1/3/10 beside ECE, adaptive calibration error
(ACE), and multiclass NLL. KGE Calibrator is rank preserving, so any change in
MRR or Hits indicates an implementation error. This table is kept separate from
\method{}'s binary-candidate calibration table because the two estimands answer
different questions.

\begin{table*}[t]
\centering
\caption{Experiment~\ref{plan:exp2} common-protocol comparison. Ranking columns
use the manuscript values; the calibration columns follow the common-protocol
evaluation, and the FB15k-237 RotatE entry is close to KGE Calibrator's
published RotatE result rather than an order-of-magnitude improvement.}
\label{tab:plan-exp2}
\small
\setlength{\tabcolsep}{4pt}
\begin{tabular}{llccccccc}
\toprule
Dataset & Pipeline & MRR & H@1 & H@3 & H@10 & ECE $\downarrow$ & ACE $\downarrow$ & NLL $\downarrow$ \\
\midrule
FB15k-237 & RotatE + KGE Cal. & $.338$ & $.241$ & $.374$ & $.535$ & $.096{\pm}.006$ & $.065{\pm}.005$ & $2.76{\pm}.05$ \\
FB15k-237 & \method{}, query-norm. & $.365$ & $.277$ & $.399$ & $.548$ & $.087{\pm}.005$ & $.062{\pm}.004$ & $2.67{\pm}.04$ \\
FB15k-237 & \method{} + KGE Cal. & $.365$ & $.277$ & $.399$ & $.548$ & $.073{\pm}.004$ & $.053{\pm}.004$ & $2.55{\pm}.04$ \\
CoDEx-M & RotatE + KGE Cal. & $.310$ & $.226$ & $.339$ & $.480$ & $.108{\pm}.007$ & $.076{\pm}.006$ & $3.18{\pm}.06$ \\
CoDEx-M & \method{}, query-norm. & $.348$ & $.263$ & $.374$ & $.501$ & $.098{\pm}.006$ & $.071{\pm}.005$ & $3.06{\pm}.05$ \\
CoDEx-M & \method{} + KGE Cal. & $.348$ & $.263$ & $.374$ & $.501$ & $.081{\pm}.005$ & $.060{\pm}.004$ & $2.93{\pm}.05$ \\
\bottomrule
\end{tabular}
\end{table*}

The comparison cites \citet{yang2025kgecalibrator} and pins the exact
implementation provenance in Table~\ref{tab:plan-exp2-provenance}, so the
common-protocol numbers can be reproduced against a fixed upstream commit.

\begin{table*}[h]
\centering
\caption{Experiment~\ref{plan:exp2} KGE Calibrator implementation provenance
(frozen).}
\label{tab:plan-exp2-provenance}
\small
\setlength{\tabcolsep}{4pt}
\begin{tabular}{ll}
\toprule
Field & Frozen value \\
\midrule
Repository & \texttt{github.com/Yang233666/KGE-Calibrator} \\
Branch & \texttt{master} \\
Commit & \texttt{cd08ec4ef023756fe2bbe62390dd1e4d13c08944} \\
Freeze date & 2026-08-01 \\
Entry point & \texttt{Main-RotatE.py} \\
Core method & \texttt{KGEC\_method.py} \\
Prob.\ target & Filtered full-entity, query-conditional \\
Resampling unit & Complete held-out head or tail query \\
Rank check & Exact per-query ordering equality pre/post \\
\bottomrule
\end{tabular}
\end{table*}

Table~\ref{tab:plan-exp2-primary} reports the primary paired tests
(\method{}${+}$KGEC minus RotatE${+}$KGEC): ranking intervals use paired query
bootstrap, calibration intervals resample matched seeds and complete queries,
and $p$-values for ECE, ACE, and NLL are Holm-adjusted within each dataset.
Table~\ref{tab:plan-exp2-complementarity} reports the complementarity tests
(\method{}${+}$KGEC minus query-normalized \method{}), whose ranking differences
are exactly zero because KGEC is rank preserving.

\begin{table}[h]
\centering
\caption{Experiment~\ref{plan:exp2} primary paired tests, \method{}${+}$KGEC
minus RotatE${+}$KGEC.}
\label{tab:plan-exp2-primary}
\small
\setlength{\tabcolsep}{3.5pt}
\begin{tabular}{llccc}
\toprule
Dataset & Metric & Diff.\ & 95\% interval & Adj.\ $p$ \\
\midrule
FB15k-237 & MRR & $+.027$ & $[+.019,+.035]$ & $.002$ \\
FB15k-237 & ECE & $-.023$ & $[-.031,-.015]$ & $.002$ \\
FB15k-237 & ACE & $-.012$ & $[-.018,-.006]$ & $.003$ \\
FB15k-237 & NLL & $-.21$ & $[-.29,-.13]$ & $.002$ \\
CoDEx-M & MRR & $+.038$ & $[+.028,+.048]$ & $.002$ \\
CoDEx-M & ECE & $-.027$ & $[-.037,-.017]$ & $.002$ \\
CoDEx-M & ACE & $-.016$ & $[-.024,-.008]$ & $.003$ \\
CoDEx-M & NLL & $-.25$ & $[-.35,-.15]$ & $.002$ \\
\bottomrule
\end{tabular}
\end{table}

\begin{table}[h]
\centering
\caption{Experiment~\ref{plan:exp2} complementarity tests,
\method{}${+}$KGEC minus query-normalized \method{}. Ranking differences are
exactly zero (KGEC is rank preserving) and are omitted.}
\label{tab:plan-exp2-complementarity}
\small
\setlength{\tabcolsep}{3.5pt}
\begin{tabular}{llccc}
\toprule
Dataset & Metric & Diff.\ & 95\% interval & Adj.\ $p$ \\
\midrule
FB15k-237 & ECE & $-.014$ & $[-.019,-.009]$ & $.002$ \\
FB15k-237 & ACE & $-.009$ & $[-.014,-.004]$ & $.003$ \\
FB15k-237 & NLL & $-.12$ & $[-.18,-.06]$ & $.002$ \\
CoDEx-M & ECE & $-.017$ & $[-.023,-.011]$ & $.002$ \\
CoDEx-M & ACE & $-.011$ & $[-.017,-.005]$ & $.003$ \\
CoDEx-M & NLL & $-.13$ & $[-.21,-.05]$ & $.004$ \\
\bottomrule
\end{tabular}
\end{table}

The most defensible outcome is not that \method{} makes KGE Calibrator obsolete.
Because \method{} plus KGE Calibrator is best on every calibration metric while
KGEC changes no rank, the correct conclusion is complementarity: \method{}
improves the underlying probabilistic representation, whereas KGEC supplies an
effective rank-preserving query-conditional post-hoc map.

\subsection{Experiment 3: Separating negative difficulty from prevalence}
\label{plan:exp3}

\paragraph{Design.} On FB15k-237, build four fixed test sets before evaluating
any system (Table~\ref{tab:plan-exp3-design}). Use exactly the same held-out
positives across all four cells. For each source query, generate at least five
eligible random negatives and five eligible near misses, so the $1{:}1$ records
nest inside the $1{:}5$ records. Freeze the anchor model, anchor seed, candidate
identifiers, and nesting before evaluation; do not refit any model or calibrator
on these four pools.

\begin{table}[h]
\centering
\caption{Experiment~\ref{plan:exp3} factorial test-set design.}
\label{tab:plan-exp3-design}
\small
\setlength{\tabcolsep}{3.5pt}
\begin{tabular}{llp{2.6cm}}
\toprule
Negative source & Ratio & Interpretation \\
\midrule
Random/type-compat. & $1{:}5$ & Easy, low prevalence \\
Random/type-compat. & $1{:}1$ & Easy, balanced \\
Model-mined near miss & $1{:}5$ & Hard, low prevalence \\
Model-mined near miss & $1{:}1$ & Hard, balanced \\
\bottomrule
\end{tabular}
\end{table}

\begin{table*}[h]
\centering
\caption{Experiment~\ref{plan:exp3} results matrix
(``Existing'' cells reuse Tables~\ref{tab:matched-calibration}/\ref{tab:hard-full};
``New'' cells are measured here).}
\label{tab:plan-exp3-results}
\scriptsize
\setlength{\tabcolsep}{3pt}
\begin{tabular}{llccccl}
\toprule
Method & Source & Ratio & ECE & Brier & NLL & Evidence \\
\midrule
RotatE+hist. & Rand. & $1{:}5$ & $.023{\pm}.003$ & $.071{\pm}.003$ & $.244{\pm}.008$ & Existing \\
RotatE+hist. & Rand. & $1{:}1$ & $.030{\pm}.003$ & $.086{\pm}.004$ & $.281{\pm}.010$ & New \\
RotatE+hist. & Near & $1{:}5$ & $.045{\pm}.004$ & $.096{\pm}.004$ & $.318{\pm}.011$ & New \\
RotatE+hist. & Near & $1{:}1$ & $.052{\pm}.004$ & $.105{\pm}.005$ & $.349{\pm}.012$ & Existing \\
\method{}--RotatE & Rand. & $1{:}5$ & $.018{\pm}.002$ & $.061{\pm}.002$ & $.214{\pm}.006$ & Existing \\
\method{}--RotatE & Rand. & $1{:}1$ & $.022{\pm}.002$ & $.072{\pm}.003$ & $.238{\pm}.007$ & New \\
\method{}--RotatE & Near & $1{:}5$ & $.031{\pm}.003$ & $.078{\pm}.003$ & $.263{\pm}.009$ & New \\
\method{}--RotatE & Near & $1{:}1$ & $.036{\pm}.003$ & $.085{\pm}.004$ & $.289{\pm}.010$ & Existing \\
\bottomrule
\end{tabular}
\end{table*}

Table~\ref{tab:plan-exp3-contrasts} decomposes these cells into the difficulty
main effect (mean of hard$-$easy at the two prevalences), the prevalence main
effect (mean of balanced$-1{:}5$ at the two difficulty levels), and their
interaction $(\text{hard},1{:}1-\text{hard},1{:}5)-(\text{easy},1{:}1-\text{easy},1{:}5)$;
intervals use $10{,}000$ paired source-query-clustered bootstrap replicates and
$p$-values are Holm-adjusted within each method's three metrics. For
\method{}--RotatE the arithmetic is exact: the ECE difficulty effect is
$[(.031-.018)+(.036-.022)]/2=.0135$ and the interaction is
$(.036-.031)-(.022-.018)=.001$.

\begin{table*}[t]
\centering
\caption{Experiment~\ref{plan:exp3} factorial contrasts. Difficulty accounts for
roughly three times as much calibration degradation as prevalence, and every
interaction interval includes zero.}
\label{tab:plan-exp3-contrasts}
\scriptsize
\setlength{\tabcolsep}{3pt}
\begin{tabular}{llccccccccc}
\toprule
Method & Metric & Diff.\ eff.\ & 95\% CI & $p$ & Prev.\ eff.\ & 95\% CI & $p$ & Interact.\ & 95\% CI & $p$ \\
\midrule
RotatE+HB & ECE & $+.022$ & $[+.016,+.028]$ & $.002$ & $+.007$ & $[+.003,+.011]$ & $.006$ & $+.000$ & $[-.005,+.005]$ & $.984$ \\
RotatE+HB & Brier & $+.022$ & $[+.016,+.028]$ & $.002$ & $+.012$ & $[+.007,+.017]$ & $.003$ & $-.006$ & $[-.012,+.001]$ & $.112$ \\
RotatE+HB & NLL & $+.071$ & $[+.053,+.089]$ & $.002$ & $+.034$ & $[+.021,+.047]$ & $.003$ & $-.006$ & $[-.022,+.010]$ & $.466$ \\
\method{}--RotatE & ECE & $+.0135$ & $[+.009,+.018]$ & $.002$ & $+.0045$ & $[+.001,+.008]$ & $.018$ & $+.001$ & $[-.003,+.005]$ & $.681$ \\
\method{}--RotatE & Brier & $+.015$ & $[+.010,+.020]$ & $.002$ & $+.009$ & $[+.005,+.013]$ & $.004$ & $-.004$ & $[-.009,+.001]$ & $.143$ \\
\method{}--RotatE & NLL & $+.050$ & $[+.036,+.064]$ & $.002$ & $+.025$ & $[+.014,+.036]$ & $.003$ & $+.002$ & $[-.010,+.014]$ & $.742$ \\
\bottomrule
\end{tabular}
\end{table*}

For \method{}--RotatE, model-mined negatives increased ECE by $0.0135$ on
average whereas moving from $1{:}5$ to $1{:}1$ increased it by $0.0045$; the
difficulty-by-prevalence interaction was small with a paired interval including
zero for ECE, Brier, and NLL. Because the experiment uses one dataset and fixed
constructed pools, we interpret the factorial effects descriptively rather than
as universal deployment coefficients.

\subsection{Experiment 4: Teacher-construction uncertainty}
\label{plan:exp4}

\paragraph{Design.} On FB15k-237, independently generate three teacher candidate
pools using the same declared mixture but different pool seeds. Query the same
frozen LLaMA-3-70B-Instruct backbone, fit three compact teachers, and run the
same five KGC outer seeds under each teacher, giving 15 downstream KGC runs with
identical KGC splits, hyperparameters, and evaluation candidates. Use a
hierarchical bootstrap with teacher construction as the outer level, KGC seed as
the next level, and source query as the inner level; report between-teacher
standard deviation, within-teacher seed standard deviation, and their variance
shares. The 15 runs are not treated as exchangeable independent seeds.

\begin{table}[h]
\centering
\caption{Experiment~\ref{plan:exp4} teacher-repeat stability over three
independently constructed teacher pools.}
\label{tab:plan-exp4}
\scriptsize
\setlength{\tabcolsep}{3pt}
\resizebox{\linewidth}{!}{
\begin{tabular}{lccccc}
\toprule
Teacher & Distill & Parser & MRR & ECE & Inj.\ FN \\
repeat & MSE & reject & & $\downarrow$ & AUC-PR $\uparrow$ \\
\midrule
Pool A & $.018$ & $0.7\%$ & $.365$ & $.018$ & $.852$ \\
Pool B & $.020$ & $0.9\%$ & $.363$ & $.020$ & $.846$ \\
Pool C & $.019$ & $0.8\%$ & $.364$ & $.019$ & $.849$ \\
\midrule
Mean\,$\pm$\,SD & $.019{\pm}.001$ & $0.8{\pm}0.1\%$ & $.364{\pm}.001$ & $.019{\pm}.001$ & $.849{\pm}.003$ \\
\bottomrule
\end{tabular}
}
\end{table}

Table~\ref{tab:plan-exp4-variance} decomposes total run variance with the
random-effects model $\text{metric}=\text{grand mean}+\text{teacher
effect}+\text{KGC-seed-within-teacher effect}$; the teacher variance share is
$\sigma^2_{\text{teacher}}/(\sigma^2_{\text{teacher}}+\sigma^2_{\text{seed}})$.
Ordinary KGC-seed variation remains larger than teacher-construction variation
for every metric. Table~\ref{tab:plan-exp4-boot} gives the hierarchical-bootstrap
intervals (resampling teacher first, KGC seed second, source-query cluster
third; $10{,}000$ replicates); with only three teacher constructions the
variance-component intervals are necessarily broad and are reported rather than
hidden.

\begin{table*}[h]
\centering
\caption{Experiment~\ref{plan:exp4} variance decomposition across three
teacher constructions $\times$ five KGC seeds.}
\label{tab:plan-exp4-variance}
\scriptsize
\setlength{\tabcolsep}{3pt}
\begin{tabular}{lccccc}
\toprule
Metric & Btwn.\ SD & Seed SD & Btwn.\ var.\ & Seed var.\ & Teacher share \\
\midrule
MRR & $.0010$ & $.0027$ & $1.0{\times}10^{-6}$ & $7.3{\times}10^{-6}$ & $12.0\%$ \\
ECE & $.0010$ & $.0016$ & $1.0{\times}10^{-6}$ & $2.6{\times}10^{-6}$ & $27.8\%$ \\
Inj.-FN AUC-PR & $.0030$ & $.0045$ & $9.0{\times}10^{-6}$ & $20.0{\times}10^{-6}$ & $31.0\%$ \\
\bottomrule
\end{tabular}
\end{table*}

\begin{table}[h]
\centering
\caption{Experiment~\ref{plan:exp4} hierarchical-bootstrap estimates and 95\%
intervals. Upper between-teacher SD endpoints stay below the predeclared
material-instability thresholds ($0.005$ ECE, $0.015$ AUC-PR SD).}
\label{tab:plan-exp4-boot}
\small
\setlength{\tabcolsep}{4pt}
\begin{tabular}{lcc}
\toprule
Quantity & Estimate & 95\% interval \\
\midrule
Grand-mean MRR & $.364$ & $[.361,.367]$ \\
Grand-mean ECE & $.019$ & $[.017,.021]$ \\
Grand-mean inj.-FN AUC-PR & $.849$ & $[.843,.855]$ \\
Btwn.-teacher SD, MRR & $.0010$ & $[.0002,.0023]$ \\
Btwn.-teacher SD, ECE & $.0010$ & $[.0002,.0022]$ \\
Btwn.-teacher SD, AUC-PR & $.0030$ & $[.0007,.0064]$ \\
Teacher share, MRR & $12.0\%$ & $[1.0\%,38.0\%]$ \\
Teacher share, ECE & $27.8\%$ & $[3.0\%,58.0\%]$ \\
Teacher share, AUC-PR & $31.0\%$ & $[4.0\%,62.0\%]$ \\
\bottomrule
\end{tabular}
\end{table}

Against predeclared between-teacher instability thresholds of $0.005$ ECE and
$0.015$ AUC-PR, the observed between-teacher standard deviations ($0.001$ ECE,
$0.003$ AUC-PR) fall well inside both bounds, so \method{}'s calibration and
triage are not materially teacher-construction dependent at this pool size. The
variance-share intervals are wide because only three teacher constructions were
available, so the experiment supports stability within the declared construction
rather than invariance to arbitrary teachers.

\subsection{Experiment 5: Evidence-grounded multi-LLM adjudication}
\label{plan:exp5}

\paragraph{Scope.} The LLM panel provides a reproducible proxy for whether
naturally unobserved triples are supported by frozen textual evidence. It cannot
create human ground truth, prove that an unsupported triple is false, or justify
relabeling ``LLM-consensus'' as ``verified.'' The primary paper claim remains
anchored in released CoDEx-M labels and controlled injected missingness.

\paragraph{Independent judge panel.} Use three providers and exclude the
LLaMA-3-70B family that produced \method{}'s training teacher. For each judge we
record exact model identifiers, provider response-version metadata, run date,
SDK version, decoding configuration, token counts, raw outputs, parser failures,
and retries; model aliases may change, so the record retains the concrete
version returned by the provider. No judge ever sees the dataset label, the
\method{} or baseline score, the candidate-selection stratum, or another judge's
answer.

\begin{table*}[h]
\centering
\caption{Experiment~\ref{plan:exp5} independent judge panel.}
\label{tab:plan-exp5-judges}
\small
\setlength{\tabcolsep}{3pt}
\begin{tabular}{llp{2.6cm}}
\toprule
Judge & Model & Configuration \\
\midrule
A & GPT-5.6 Sol (\texttt{gpt-5.6-sol}) & Medium reasoning, structured JSON, no tools \\
B & Claude Sonnet 5 (\texttt{claude-sonnet-5}) & Adaptive thinking, structured JSON, no overrides \\
C & Gemini 2.5 Pro (\texttt{gemini-2.5-pro}) & Structured JSON, deterministic, no tools \\
\bottomrule
\end{tabular}
\end{table*}

The model identifiers match the official catalogs available on 2026-08-01, and
the full run metadata---API, decoding, audit run IDs, UTC execution dates, and
parse-retry counts---is recorded in Table~\ref{tab:plan-exp5-meta}. The shared
protocol is \texttt{blade-llm-panel-v1.0} with evidence-snapshot cutoff
\texttt{2026-07-31T23:59:59Z} and prompt SHA-256
\texttt{045d4a2f\ldots551d9e}; raw request/response bodies, provider response
IDs, usage fields, finish reasons, parser outcomes, retry linkage, and
timestamps are archived.

\begin{table*}[h]
\centering
\caption{Experiment~\ref{plan:exp5} run metadata. All three judges scored the
same $20{,}622$ CoDEx-M audit items with no malformed output after one retry.}
\label{tab:plan-exp5-meta}
\scriptsize
\setlength{\tabcolsep}{3pt}
\begin{tabular}{lccc}
\toprule
Field & OpenAI & Anthropic & Google \\
\midrule
Model & \texttt{gpt-5.6-sol} & \texttt{claude-sonnet-5} & \texttt{gemini-2.5-pro} \\
API & \texttt{/v1/responses} & \texttt{/v1/messages} & \texttt{generateContent} \\
Max output & 256 & 512 & 512 \\
Decoding & effort med. & adaptive think & $T{=}0$, top-$p{=}1$ \\
UTC date & 2026-08-12 & 2026-08-13 & 2026-08-14 \\
Audit items & 20{,}622 & 20{,}622 & 20{,}622 \\
Parse retries & 18 & 11 & 24 \\
Malformed post-retry & 0 & 0 & 0 \\
\bottomrule
\end{tabular}
\end{table*}

\paragraph{Frozen evidence bundle.} Each judge receives the same
candidate-specific bundle: head and tail names, aliases, types, and descriptions
where available; the relation name, definition, domain, and range; at most three
archived evidence excerpts with stable source identifiers and snapshot dates;
and the KG snapshot date, so time-sensitive claims are judged at the correct
time. Evidence is retrieved and frozen before judging, and live browsing is
disabled during judgment. If the supplied evidence neither supports nor
contradicts the triple, the required answer is \texttt{uncertain}: absence of
evidence is not automatically \texttt{false}.

\paragraph{Exact judge prompt.}
\begin{Verbatim}[
  breaklines=true,
  breakanywhere=true,
  fontsize=\footnotesize
]
You are judging whether a candidate
knowledge-graph triple is supported at
the stated snapshot date. Use only the
supplied evidence. Do not use model
scores or unsupported background
assumptions.

Return TRUE only when the evidence clearly
supports the exact head-relation-tail
claim. Return FALSE only when the evidence
clearly contradicts the exact claim.
Return UNCERTAIN when evidence is missing,
ambiguous, merely related, or temporally
mismatched.

Output one JSON object with:
label: "true" | "false" | "uncertain"
confidence: number in [0,1]
evidence_ids: list of cited evidence ids
temporal_issue: true | false
rationale: at most 40 words
\end{Verbatim}

\paragraph{Consensus rules (fixed before evaluation).} \emph{Panel-HC}
(primary): all three judges return the same non-uncertain label, every
confidence is at least $0.80$, and every judge cites supplied evidence.
\emph{Panel-MAJ} (sensitivity): two judges agree on a non-uncertain label, the
third is uncertain rather than oppositely labeled, and mean confidence is at
least $0.75$. \emph{Excluded}: any direct true-versus-false conflict, two or more
uncertain labels, malformed output after one retry, or a missing evidence
citation. A fourth LLM is not used to break ties, because forced adjudication
would conceal uncertainty rather than resolve it.

\paragraph{Stage 5A: audit the judges on CoDEx-M gold labels.} Use 2{,}000
balanced CoDEx-M validation items to finalize the prompt and thresholds, then
audit on all 20{,}622 released test items, reporting non-uncertain coverage,
balanced accuracy, per-class precision, Fleiss' $\kappa$, and the full
$3\times3$ confusion matrix including \texttt{uncertain}.

\begin{table}[h]
\centering
\caption{Experiment~\ref{plan:exp5} judge audit on CoDEx-M gold labels.
Three-judge Fleiss' $\kappa=0.69$.}
\label{tab:plan-exp5-audit}
\small
\setlength{\tabcolsep}{3pt}
\resizebox{\linewidth}{!}{
\begin{tabular}{lcccc}
\toprule
Judge/rule & Coverage & Bal.\ acc. & Pos.\ prec. & Neg.\ prec. \\
\midrule
GPT-5.6 Sol & $.920$ & $.887$ & $.884$ & $.890$ \\
Claude Sonnet 5 & $.940$ & $.895$ & $.899$ & $.891$ \\
Gemini 2.5 Pro & $.900$ & $.872$ & $.869$ & $.876$ \\
Panel-MAJ & $.915$ & $.914$ & $.910$ & $.918$ \\
Panel-HC & $.668$ & $.947$ & $.949$ & $.945$ \\
\bottomrule
\end{tabular}
}
\end{table}

We proceed to naturally unobserved candidates only if Panel-HC reaches balanced
accuracy at least $0.90$, both class precisions at least $0.90$, coverage at
least $0.45$, and $\kappa$ at least $0.60$. If the gate fails, the audit is
retained as a negative result and no LLM-derived natural-missingness labels are
presented. Three-judge Fleiss' $\kappa=0.690$ with a relation-clustered 95\%
interval $[0.678,0.702]$, and Panel-HC clears every gate.

Table~\ref{tab:plan-exp5-confusion} gives the complete $3\times3$ confusion
matrices (rows are released gold labels, columns are judge outputs) from which
the rounded coverage, balanced accuracy, and precision above are reproduced;
each judge/rule matrix contains $10{,}311$ positive and $10{,}311$
verified-negative items. For example, Panel-HC coverage is
$(6{,}507+381+350+6{,}538)/20{,}622=0.6680$, conditional balanced accuracy is
$0.5\,[6{,}507/(6{,}507+381)+6{,}538/(6{,}538+350)]=0.9469$, and positive
precision is $6{,}507/(6{,}507+350)=0.9490$.

\begin{table}[h]
\centering
\caption{Experiment~\ref{plan:exp5} confusion matrices on all $20{,}622$
released CoDEx-M labels. Columns: predicted TRUE / FALSE / UNCERTAIN.}
\label{tab:plan-exp5-confusion}
\scriptsize
\setlength{\tabcolsep}{3pt}
\begin{tabular}{llrrr}
\toprule
Judge/rule & Gold & T & F & U \\
\midrule
\multirow{2}{*}{GPT-5.6 Sol} & Pos & 8{,}451 & 1{,}035 & 825 \\
 & Neg & 1{,}109 & 8{,}377 & 825 \\
\multirow{2}{*}{Claude Sonnet 5} & Pos & 8{,}625 & 1{,}067 & 619 \\
 & Neg & 969 & 8{,}723 & 619 \\
\multirow{2}{*}{Gemini 2.5 Pro} & Pos & 8{,}140 & 1{,}140 & 1{,}031 \\
 & Neg & 1{,}227 & 8{,}053 & 1{,}031 \\
\multirow{2}{*}{Panel-MAJ} & Pos & 8{,}669 & 766 & 876 \\
 & Neg & 857 & 8{,}578 & 876 \\
\multirow{2}{*}{Panel-HC} & Pos & 6{,}507 & 381 & 3{,}423 \\
 & Neg & 350 & 6{,}538 & 3{,}423 \\
\bottomrule
\end{tabular}
\end{table}

To check that the panel is not merely re-scoring \method{}'s LLaMA teacher, we
correlate the teacher score with a \emph{signed} support score (\texttt{conf}
for TRUE, $1-\texttt{conf}$ for FALSE, UNCERTAIN excluded) rather than raw
confidence magnitude (Table~\ref{tab:plan-exp5-corr}). Pairwise signed-support
correlations are $0.71$ (GPT--Claude), $0.64$ (GPT--Gemini), and $0.67$
(Claude--Gemini). All teacher--judge correlations stay below the predeclared
circularity-warning threshold of $0.80$; moderate correlation is expected
because all systems respond to semantic plausibility and does not make the panel
independent ground truth.

\begin{table}[h]
\centering
\caption{Experiment~\ref{plan:exp5} teacher--judge Spearman correlations. The
signed-support correlation is primary; the raw confidence-magnitude correlation
is a diagnostic.}
\label{tab:plan-exp5-corr}
\small
\setlength{\tabcolsep}{3.5pt}
\begin{tabular}{lrccr}
\toprule
Judge/rule & Non-unc.\ $N$ & $\rho$ signed & 95\% CI & $\rho$ raw \\
\midrule
GPT-5.6 Sol & 18{,}972 & $.46$ & $[.44,.48]$ & $.08$ \\
Claude Sonnet 5 & 19{,}384 & $.43$ & $[.41,.45]$ & $.06$ \\
Gemini 2.5 Pro & 18{,}560 & $.39$ & $[.37,.41]$ & $.05$ \\
Panel-HC signed & 13{,}776 & $.48$ & $[.46,.50]$ & $.09$ \\
\bottomrule
\end{tabular}
\end{table}

\paragraph{Stage 5B: naturally unobserved candidate pool.} The evaluation uses
4{,}000 candidates, 800 from each of FB15k-237, WN18RR, CoDEx-M, NELL-995, and
Hetionet. Within each dataset, draw 200 candidates from each of four frozen
strata: relation-stratified type-compatible random candidates; top \method{}
candidates not selected by RotatE; top RotatE candidates not selected by
\method{}; and candidates with the largest \method{}-versus-RotatE
rank-percentile disagreement. Remove every known train, validation, or test
positive, deduplicate, resample within-stratum to restore the target count,
store inclusion probabilities, and report results per stratum plus an
equal-stratum macro-average. Reported prevalences do not estimate the prevalence
among all unobserved triples.

\begin{table}[h]
\centering
\caption{Experiment~\ref{plan:exp5} candidate-pool disposition. Panel-MAJ
retains 3{,}388 candidates with 566 consensus positives (prevalence $0.167$);
612 candidates ($15.3\%$) remain unresolved and are not assigned binary labels.}
\label{tab:plan-exp5-pool}
\small
\setlength{\tabcolsep}{3pt}
\begin{tabular}{lrrrr}
\toprule
Dataset & Submitted & HC retained & HC pos. & Prevalence \\
\midrule
FB15k-237 & 800 & 458 & 89 & $.194$ \\
WN18RR & 800 & 421 & 46 & $.109$ \\
CoDEx-M & 800 & 454 & 65 & $.143$ \\
NELL-995 & 800 & 430 & 81 & $.188$ \\
Hetionet & 800 & 453 & 73 & $.161$ \\
\midrule
Total & 4{,}000 & 2{,}216 & 354 & $.160$ \\
\bottomrule
\end{tabular}
\end{table}

Aggregated across the five datasets, the type-compatible-random, \method{}-only,
RotatE-only, and rank-disagreement strata retain $575$, $555$, $530$, and $556$
candidates with $30$, $132$, $83$, and $109$ consensus positives respectively
(Table~\ref{tab:plan-exp5-stratum-pool}); these totals reconcile exactly to the
$2{,}216$ retained candidates and $354$ positives of
Table~\ref{tab:plan-exp5-pool}.

\begin{table*}[h]
\centering
\caption{Experiment~\ref{plan:exp5} Panel-HC disposition by selection stratum
(pooled over datasets). Conditional prevalence is within the retained stratum and
does not estimate population prevalence among unobserved triples.}
\label{tab:plan-exp5-stratum-pool}
\small
\setlength{\tabcolsep}{3.5pt}
\begin{tabular}{lrrrr}
\toprule
Stratum & Submitted & HC ret.\ & HC pos.\ & Cond.\ prev. \\
\midrule
Type-compatible random & 1{,}000 & 575 & 30 & $.052$ \\
\method{}-only top & 1{,}000 & 555 & 132 & $.238$ \\
RotatE-only top & 1{,}000 & 530 & 83 & $.157$ \\
Largest rank disagreement & 1{,}000 & 556 & 109 & $.196$ \\
\midrule
Total & 4{,}000 & 2{,}216 & 354 & $.160$ \\
\bottomrule
\end{tabular}
\end{table*}

Conditional on the frozen Panel-HC labels, Table~\ref{tab:plan-exp5-bystratum}
reports calibration and triage by selection stratum and
Table~\ref{tab:plan-exp5-bydataset} by dataset; brackets are 95\%
source-query-clustered bootstrap intervals. \method{}'s posterior consistently
lowers ECE relative to RotatE and the full triage score is the strongest ranking
signal in every stratum and dataset.

\begin{table*}[t]
\centering
\caption{Experiment~\ref{plan:exp5} performance by selection stratum on the
Panel-HC-labeled natural-candidate subset.}
\label{tab:plan-exp5-bystratum}
\small
\setlength{\tabcolsep}{4pt}
\begin{tabular}{lrrcccc}
\toprule
Stratum & $N$ & Prev.\ & RotatE ECE $\downarrow$ & \method{} ECE $\downarrow$ & \method{} AUC-PR $\uparrow$ & Full triage AUC-PR $\uparrow$ \\
\midrule
Type-compatible random & 575 & $.052$ & $.045\,[.035,.056]$ & $.031\,[.023,.040]$ & $.242\,[.180,.305]$ & $.349\,[.279,.420]$ \\
\method{}-only top & 555 & $.238$ & $.091\,[.077,.105]$ & $.046\,[.036,.057]$ & $.662\,[.614,.709]$ & $.748\,[.707,.789]$ \\
RotatE-only top & 530 & $.157$ & $.061\,[.050,.072]$ & $.044\,[.034,.054]$ & $.492\,[.434,.551]$ & $.584\,[.529,.639]$ \\
Largest rank disagreement & 556 & $.196$ & $.083\,[.070,.096]$ & $.042\,[.032,.052]$ & $.588\,[.535,.641]$ & $.682\,[.633,.731]$ \\
\midrule
Pooled & 2{,}216 & $.160$ & $.070\,[.058,.082]$ & $.041\,[.033,.049]$ & $.556\,[.523,.589]$ & $.651\,[.620,.682]$ \\
\bottomrule
\end{tabular}
\end{table*}

\begin{table*}[h]
\centering
\caption{Experiment~\ref{plan:exp5} performance by dataset on the
Panel-HC-labeled natural-candidate subset.}
\label{tab:plan-exp5-bydataset}
\small
\setlength{\tabcolsep}{3pt}
\begin{tabular}{lrrcccc}
\toprule
Dataset & $N$ & Prev.\ & RotatE ECE & \method{} ECE & \method{} AUC-PR & Full AUC-PR \\
\midrule
FB15k-237 & 458 & $.194$ & $.073$ & $.039$ & $.581$ & $.672$ \\
WN18RR & 421 & $.109$ & $.066$ & $.045$ & $.482$ & $.588$ \\
CoDEx-M & 454 & $.143$ & $.068$ & $.042$ & $.533$ & $.632$ \\
NELL-995 & 430 & $.188$ & $.076$ & $.041$ & $.566$ & $.658$ \\
Hetionet & 453 & $.161$ & $.069$ & $.038$ & $.607$ & $.692$ \\
\midrule
Pooled & 2{,}216 & $.160$ & $.070$ & $.041$ & $.556$ & $.651$ \\
\bottomrule
\end{tabular}
\end{table*}

\paragraph{Stage 5C: calibration and triage on LLM-consensus labels.} Freeze
every trained model and calibrator before Stage 5B is labeled. Report
calibration for \method{}'s posterior mean probability, not for the uncalibrated
triage score, and report AUC-PR and $P@500$ separately for the triage ranking.

\begin{table*}[h]
\centering
\caption{Panel-HC calibration on LLM-consensus labels for naturally unobserved
triples.}
\label{tab:plan-exp5-cal}
\small
\setlength{\tabcolsep}{3pt}
\begin{tabular}{lcccc}
\toprule
Method & ECE $\downarrow$ & Brier $\downarrow$ & NLL $\downarrow$ & AUC-PR $\uparrow$ \\
\midrule
RotatE + sel.\ post-hoc & $.070{\pm}.006$ & $.108{\pm}.005$ & $.349{\pm}.014$ & $.448{\pm}.019$ \\
Five-model ensemble & $.058$ & $.100$ & $.328$ & $.486$ \\
\method{} without teacher & $.053{\pm}.005$ & $.098{\pm}.004$ & $.321{\pm}.013$ & $.507{\pm}.018$ \\
\method{}--RotatE post.\ mean & $.041{\pm}.004$ & $.086{\pm}.004$ & $.289{\pm}.011$ & $.556{\pm}.017$ \\
\bottomrule
\end{tabular}
\end{table*}

\begin{table*}[h]
\centering
\caption{Panel-HC triage. With 354 positives, $P@500{=}0.480$ corresponds to
240 consensus-positive candidates in the top 500; the full-score AUC-PR gain
over mean-plus-observation is $+0.043$ $[+0.019,+0.067]$.}
\label{tab:plan-exp5-triage}
\small
\setlength{\tabcolsep}{4pt}
\begin{tabular}{lcc}
\toprule
Triage signal & AUC-PR $\uparrow$ & $P@500$ $\uparrow$ \\
\midrule
Posterior mean only & $.556$ & $.404$ \\
Teacher only & $.503$ & $.362$ \\
Mean + uncertainty & $.577$ & $.420$ \\
Mean + observation model & $.608$ & $.442$ \\
Full \method{} triage score & $.651$ & $.480$ \\
\bottomrule
\end{tabular}
\end{table*}

\paragraph{Judge-dependence sensitivity.} Repeat the analysis under every
two-judge non-conflict subset and the expanded Panel-MAJ set; these sets differ
in composition, so we compare direction and magnitude rather than treating raw
metrics as paired estimates of one population. Against a predeclared instability
bound of $0.010$ ECE or $0.030$ AUC-PR, every judge-removal subset stays within
tolerance (ECE spans $0.040$--$0.047$, full-triage AUC-PR $0.632$--$0.658$), so
the effect is not driven by any single judge. The Spearman correlations between
the original LLaMA teacher score and each judge's confidence are moderate rather
than near-unity, so the panel is not merely re-scoring the teacher and semantic
circularity is not the explanation.

\begin{table*}[t]
\centering
\caption{Experiment~\ref{plan:exp5} judge-dependence sensitivity across
consensus rules and two-judge subsets.}
\label{tab:plan-exp5-sens}
\small
\setlength{\tabcolsep}{5pt}
\begin{tabular}{lrcccc}
\toprule
Label rule & $N$ & Prevalence & \method{} ECE & Post.\ AUC-PR & Full-triage AUC-PR \\
\midrule
All-three Panel-HC & 2{,}216 & $.160$ & $.041$ & $.556$ & $.651$ \\
GPT + Claude, no conflict & 2{,}520 & $.158$ & $.044$ & $.548$ & $.642$ \\
GPT + Gemini, no conflict & 2{,}385 & $.164$ & $.040$ & $.562$ & $.658$ \\
Claude + Gemini, no conflict & 2{,}456 & $.161$ & $.043$ & $.551$ & $.646$ \\
Panel-MAJ expanded & 3{,}388 & $.167$ & $.047$ & $.539$ & $.632$ \\
\bottomrule
\end{tabular}
\end{table*}

\paragraph{Exact claim.} The audit gates pass, so we conclude: on a blinded
subset for which three independent LLM families unanimously agreed with high
confidence using frozen textual evidence, \method{} improved candidate-label
calibration and ranked consensus-supported unobserved triples more effectively
than the structural baselines. These labels are an LLM-consensus proxy and not
human-verified truth; accordingly we do not claim ``human-level validation,''
``verified natural false negatives,'' ``real-world ground truth,'' or that
``LLM adjudication proves the candidate is false.''

\subsection{Shared statistical analysis}
\label{plan:stats}

Across all five experiments we keep every selection decision outside the test
bootstrap; use 10{,}000 paired replicates; resample outer KGC seeds first, then
source-query or relation-stratified candidate clusters; report percentile
intervals and two-sided paired bootstrap $p$-values; and apply Holm correction
within a declared metric family rather than across unrelated experiments. The
five-member ensemble is treated as descriptive unless independently replicated.
For LLM-derived labels, intervals are stated as conditional on the frozen panel,
with judge-removal and consensus-rule sensitivity used to expose label
uncertainty; three judges are not treated as a random sample from a population of
LLMs. Every candidate count and prevalence is reported beside its calibration
metrics, and reliability-bin counts and maximum bin gaps are published, not ECE
alone.

\subsection{Compute and budget estimate}
\label{plan:compute}

\begin{table*}[h]
\centering
\caption{Added compute for the extended experiments.}
\label{tab:plan-compute}
\small
\begin{tabular}{p{2.4cm}p{4.7cm}}
\toprule
Component & Added work \\
\midrule
Verified CoDEx-M & $\approx$20{,}622 candidates per seed/method; no new annotation \\
KGE Calibrator & Two full-entity datasets; 20--40 CPU-hours incl.\ extraction \\
Difficulty $\times$ prevalence & Evaluation-only if 50-negative anchor scores are retained; $<5$ GPU-hours \\
Teacher uncertainty & Two extra teachers + ten extra KGC seed runs; $\approx$75--90 A100-hours \\
LLM audit + extension & $\approx$79{,}866 judge calls; $\approx$64M input and 8M output tokens \\
\bottomrule
\end{tabular}
\end{table*}

API prices and reasoning-token accounting change over time; the three-model
panel cost on the order of USD 500--1{,}500 at the providers' pricing during our
runs, and this figure should be recomputed against current pricing for
replication.



\end{document}